\documentclass[11pt,a4paper,logo,copyright]{pluralisresearchiclr}

\usepackage{float}
\usepackage{amssymb}
\usepackage{algorithm}
\usepackage{algpseudocode}
\usepackage{amsmath}
\usepackage{amsthm} 
\usepackage{wrapfig}

\usepackage{multirow}
\usepackage{tabularx}
\usepackage{graphicx}    %
\usepackage[most]{tcolorbox}

\definecolor{tracebg}{HTML}{F7FAFC}
\definecolor{traceframe}{HTML}{2C5282}
\definecolor{thinkbg}{HTML}{FFFAF0}
\definecolor{thinkframe}{HTML}{B7791F}
\definecolor{answerbg}{HTML}{F0FFF4}
\definecolor{answerframe}{HTML}{276749}

\newtcolorbox{traceBox}[1]{%
  enhanced,
  colback=tracebg,
  colframe=traceframe,
  fonttitle=\bfseries\small,
  title={#1},
  arc=2pt,
  boxrule=0.6pt,
  left=6pt, right=6pt, top=4pt, bottom=4pt,
  breakable
}

\newtcolorbox{thinkBox}{%
  enhanced,
  colback=thinkbg,
  colframe=thinkframe,
  arc=1pt,
  boxrule=0.4pt,
  left=4pt, right=4pt, top=3pt, bottom=3pt,
  breakable,
  before skip=4pt, after skip=4pt
}

\newtcolorbox{answerBox}{%
  enhanced,
  colback=answerbg,
  colframe=answerframe,
  arc=1pt,
  boxrule=0.4pt,
  left=4pt, right=4pt, top=3pt, bottom=3pt,
  breakable,
  before skip=4pt, after skip=4pt
}

\usepackage{xcolor}
\usepackage{colortbl}
\usepackage{booktabs}
\usepackage{float}     %

\definecolor{deltaneg}{HTML}{9B2C2C}
\definecolor{groupband}{HTML}{EDF2F7}

\definecolor{headerblue}{HTML}{2C5282}
\definecolor{rowgray}{HTML}{F7FAFC}
\definecolor{deltapos}{HTML}{276749}

\definecolor{headerblue}{HTML}{2C5282}
\definecolor{rowgray}{HTML}{F7FAFC}
\definecolor{ourrow}{HTML}{C6F6D5}
\definecolor{deltapos}{HTML}{276749}
\definecolor{failcell}{HTML}{FED7D7}

\newtheorem{assumption}{Assumption}
\newtheorem{theorem}{Theorem}
\newtheorem{lemma}{Lemma}
\theoremstyle{remark}
\newtheorem{remark}{Remark}
\theoremstyle{plain}
\newtheorem{proposition}{Proposition}
\newtheorem{corollary}{Corollary}

\usepackage[utf8]{inputenc} %
\usepackage[T1]{fontenc}    %
\usepackage{url}            %
\usepackage{booktabs}       %
\usepackage{amsfonts}       %
\usepackage{nicefrac}       %
\usepackage{microtype}      %
\usepackage{xcolor}         %

\title{Communication-Efficient LLM Adaptation over Decentralized GPU Meshes}

\pdftrailerid{redacted}
\renewcommand{\today}{}
\renewcommand{\copyrightext}{\footerfont\textcopyright\, 2026 Pluralis Research. All rights reserved.}
\correspondingauthor{sameera@pluralis.ai}
\paperurl{}
\reportnumber{}
\author[1]{Sameera~Ramasinghe}
\author[1]{Shamane~Siriwardhana}
\author[1]{Thalaiyasingam~Ajanthan}
\author[1]{Hadi~Mohaghegh~Dolatabadi}
\author[1]{Chamin~P~Hewa~Koneputugodage}
\author[1]{Gil~Avraham}
\author[1]{Violetta~Shevchenko}
\author[1]{James~Snewin}
\author[1]{Karol~Pajak}
\author[1]{Harry~Xi}
\author[1]{Alexander~Long}
\affil[1]{Pluralis Research}
\hypersetup{
  pdftitle={Communication-Efficient LLM Adaptation over Decentralized GPU Meshes},
  pdfauthor={Sameera Ramasinghe, Shamane Siriwardhana, Thalaiyasingam Ajanthan, Hadi Mohaghegh Dolatabadi, Chamin P Hewa Koneputugodage, Gil Avraham, Violetta Shevchenko, James Snewin, Karol Pajak, Harry Xi, Alexander Long}
}

\begin{abstract}
Decentralized training enables large-model training over low-end GPUs
and internet-grade connections, but communication along both
data-parallel and pipeline-parallel axes becomes the primary
bottleneck. We study post-pretraining adaptation in this setting. We propose an asynchronous two-circuit
system: a fast compressed training circuit drives throughput using
activation masking for pipeline-parallel (PP) transfer and compressed
data-parallel (DP) synchronization, while a slow anchor circuit runs
occasional unmasked forward--backward passes off the critical path.
Then, we introduce a spectral correction optimizer that uses these
delayed \emph{anchor priors} to denoise masked gradients without blocking the
fast stream. Although
prior work has found aggressive activation compression unreliable, we
show that masking supports post-pretraining adaptation at high
compression rates when anchored this way. Pipeline-parallel
compression alone yields up to a $9\times$ throughput gain, and
combining it with data-parallel compression increases beyond
$40\times$ over internet-grade $\sim 200$Mbps connections, while
matching dense uncompressed performance across domain adaptation and
continual pretraining.
\end{abstract}

\begin{document}

\maketitle

\section{Introduction}

Training and adapting frontier models typically requires tightly connected datacenter clusters, putting modern AI research out of reach for smaller labs, independent researchers, and open communities \citep{narayanan2021megatron, grattafiori2024llama3}. Decentralized training offers an alternative: computation can be
distributed over the Internet across heterogeneous consumer-grade
GPUs, allowing participants to contribute otherwise fragmented
compute~\citep{diskin2021dedloc, ryabinin2023swarm, yuan2022decentralized, douillard2023diloco, jaghouar2024opendiloco, ajanthan2026asyncmesh, ramasinghemixtures}. We focus on decentralized post-pretraining adaptation, \textit{i.e.}, fine-tuning or continual pretraining open-weight models over such  GPU meshes.

This setting is fundamentally communication-limited. Commodity Internet links provide bandwidth on the order of hundreds of Mbps, compared to hundreds of Gbps in datacenter interconnects. A practical system must therefore reduce communication along both distributed-training axes: data parallelism (DP), where replicas synchronize weight gradients, and pipeline parallelism (PP), where model shards exchange activations and activation gradients.

DP communication can be reduced with gradient compression  \citep{vogels2019powersgd, wang2018atomo, basu2020qsparse, zhao2024galore} or sparse synchronization \citep{douillard2023diloco, douillard2025streamingdiloco, zhang2025demo, jaghouar2024opendiloco, sani2024photon}, but these methods still assume that each worker or replica group can host a full model. PP addresses the complementary regime: it enables models that do \emph{not} fit on a single low-end GPU by partitioning them across layers. However, PP compression is substantially more delicate. Standard compression methods such as sparsification, quantization, and low-rank projection often degrade performance at high compression rates due to error accumulation across layers \citep{rudakov2023activations, ramasinghe2025subspace}. Recent methods such as Subspace Networks \citep{ramasinghe2025subspace} and Pufferfish \citep{wang2021pufferfish} mitigate this during pretraining by imposing structural constraints, but these constraints make them difficult to apply to arbitrary pretrained checkpoints.

We therefore revisit and refine a simple but previously discarded idea: \emph{activation masking}. Our key observation is that adaptation differs from pretraining. During post-pretraining adaptation, weights evolve slowly from a strong pretrained initialization, and the goal is often to find a useful update direction rather than learn representations from scratch. In contrast to the pessimistic view in prior works \citep{ramasinghe2025subspace, rudakov2023activations}, we show that sparse activation signals are sufficient for the main training stream, provided they are anchored by higher-fidelity information.
We realize this idea with a two-circuit decentralized system. A fast
compressed circuit performs most updates using in-graph activation
masking along the PP axis and compressed synchronization along the DP
axis. In parallel, a slower anchor circuit occasionally runs full
unmasked forward and backward passes. These unmasked passes do not
block the fast circuit; instead, their gradients arrive asynchronously
as \emph{Anchor Priors}, which denoise and steer the masked
optimization trajectory. The asynchrony is principled: fine-tuning
gradients live in a low-dimensional \citep{li2018measuring, aghajanyan2021intrinsic, gur2018gradient}, slowly-drifting subspace, so a
stale unmasked gradient still points roughly the right way \citep{neyshabur2020transferred, wortsman2022robust}. This is akin to the 
same structural property that makes model merging effective \citep{wortsman2022robust}, applied
here at the gradient level.

To use these delayed priors effectively, we introduce a spectral
correction optimizer. It maintains a slow-moving momentum buffer over
anchor gradients and uses its principal directions to reweight masked
updates: anchor-supported eigendirections are reinforced, while
unsupported noisy directions are damped. We complement this design
with a convergence analysis identifying when the correction provably
helps: when fine-tuning gradients lie in a low-dimensional subspace
and the optimization trajectory drifts slowly, the anchor basis
remains informative despite staleness. The analysis also clarifies why
random PRF masking is essential; it produces an approximately unbiased
gradient estimator that the spectral filter can denoise, while top-$K$
masking introduces a structured bias the filter cannot remove.
Empirically, the system is robust to substantial staleness, and
significantly superior to top-$K$, reinforcing these predictions.

Across domain adaptation, continual pretraining, and reasoning-oriented training, our method matches or exceeds dense uncompressed training while substantially improving throughput over low-bandwidth links. PP compression alone yields up to a $9\times$ throughput gain, and combining it with DP compression increases this over $40\times$. The method generalizes across model families, depths, widths, and mesh sizes, and produces models that are more robust to pruning. Overall, our contributions are as follows:
\begin{itemize}
    \item We propose a two-circuit decentralized training system in
    which a fast masked circuit trains continuously while a slower
    unmasked anchor circuit asynchronously supplies high-fidelity
    gradient priors, decoupling throughput from the cost of clean
    gradient signals.

    \item We introduce a spectral correction optimizer that uses delayed Anchor Priors to denoise masked gradients. Our convergence analysis and empirical results demonstrate its efficacy.

    \item We show superior performance among
    competing methods, matching 
    uncompressed adaptation while achieving over a $40\times$
    throughput. Further, models trained with our method are more robust
    to pruning, improving their suitability for resource-constrained
    inference.
\end{itemize}

\section{Related Works}

\paragraph{Data parallel compression.}
Decentralized DP training has motivated two lines of work: 1) \emph{Sparse synchronization methods} reduce the number of
all-reduces by performing many local optimizer steps between outer
updates \citep{douillard2023diloco, douillard2025streamingdiloco, zhang2025demo, jaghouar2024opendiloco, sani2024photon}; we use
Streaming DiLoCo \citep{douillard2025streamingdiloco}, which overlaps
the outer synchronization with local computation. 2) \emph{Low-rank gradient compressors} reduce 
communication volume by transmitting low-rank approximations of the
gradient \citep{vogels2019powersgd, wang2018atomo, basu2020qsparse, zhao2024galore}; we combine PowerSGD \citep{vogels2019powersgd} with Streaming DiLoCo. However, the proposed design can be integrated with other DP alternatives.  \textbf{Pipeline-parallel activation compression.} \emph{Quantization-based methods} cast activations
and activation gradients to lower-precision formats with error
compensation: AQ-SGD \citep{wang2022aqsgd} compresses activation
\emph{changes} between epochs to enable error feedback over slow
networks, and TAH-Quant \citep{he2025tahquant} extends this with
tile-wise adaptive Hadamard quantization. These methods operate
\emph{off-graph}, transmitting compressed payloads that are
decompressed before use, which decouples forward and backward sparsity
patterns and requires storing per-example state across epochs to
support error feedback. We instead use \emph{in-graph} random masking
with no error feedback. Our masking is also
substantially more aggressive ($95\%$ vs.\ $4\times$/$8\times$
quantization). \textbf{Parameter-efficient fine-tuning.} reduces training VRAM by using low-rank adapters
\citep{hu2022lora, zhao2024galore, lialin2023relora, li2021prefix, lester2021power, houlsby2019adapter, pfeiffer2021adapterfusion}, but the forward and backward passes still transmit \emph{full} dense
activations across pipeline-parallel boundaries. Thus, they do not
address the communication bottleneck. 
Our method composes cleanly with PEFT
(Section~\ref{sec:peft_orthogonal}).

\section{Methodology}
\label{sec:method}

We consider decentralized post-pretraining adaptation over a
low-bandwidth GPU mesh across DP and PP dimensions. We represent a decentralized mesh as a 2D $X \times Y$ grid, where $X$ is the data-parallel (DP) dimension and $Y$ is the pipeline-parallel (PP) dimension. Each worker stores only a subset of the model layers, and workers communicate over low-bandwidth links along both DP and PP axes. For all experiments, we partition the mesh into two asynchronous circuits: an $(X-Z)\times Y$ fast masked circuit and a $Z \times Y$ unmasked \emph{anchor circuit}. In all our experiments, we fix $Z=1$. 

\begin{figure}[tbp]

    \centering
    \includegraphics[width=0.48\textwidth]{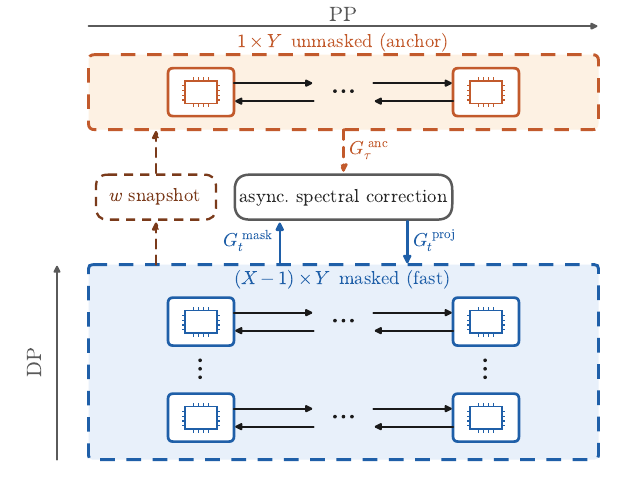}
    \caption{\textbf{System overview.} A small unmasked anchor circuit (orange, $1{\times}Y$
PP stages) asynchronously supplies clean gradients $G_\tau^{\mathrm{anc}}$
to a spectral corrector, which denoises the masked gradients
$G_t^{\mathrm{mask}}$ from the fast circuit (blue, $(X{-}1){\times}Y$
DP$\times$PP) into corrected updates $G_t^{\mathrm{proj}}$. Slow circuit periodically pulls fresh weight snapshots from the fast circuit.}
    \label{fig:system}
    
\end{figure}

The anchor circuit occasionally runs unmasked
forward--backward passes and supplies high-fidelity gradient signals
that the fast circuit uses to denoise its compressed updates. The
anchor circuit does not block the fast circuit: its gradients are
consumed asynchronously when they arrive, so the fast circuit's
throughput is set by masked communication only. This split is the core
design move. Activation masking alone reduces communication but
introduces noise that compounds across boundaries; running unmasked
passes synchronously would eliminate the throughput gain. Asynchronous
anchors give us both,\textit{ i.e., }masked-step throughput, plus a slow
correctional signal that does not enter the critical path.

\subsection{Activation Masking for PP Compression}
\label{sec:pp-masking}

Let $h \in \mathbb{R}^{T\times H}$ denote the hidden-state tensor at a
masked inter-layer boundary, where $T$ is sequence length and $H$ is
hidden size. On masked steps we apply elementwise masking
$\tilde{h} = h \odot m$ with $m\in\{0,1\}^{T\times H}$ at masking
probability $p$, retaining $K=\mathrm{round}((1-p)H)$ dimensions per
token.

\paragraph{Shared PRF avoids transmitting the mask.}
Naively, masking would require sending the retained indices for every
token alongside the values, halving the compression benefit. We avoid
this by generating $m$ from a shared deterministic pseudorandom
function $M = g(\ell, t, B, S, H)$ keyed on layer index, optimization
step, and batch metadata. Sender and receiver reconstruct the same
mask locally; only the $K$ retained values are transmitted.

\paragraph{In-graph masking gives matched forward/backward sparsity.}
Because masking is applied in-graph as a multiplication, the backward
pass inherits the same pattern automatically:
\(
\partial \mathcal{L}/\partial h_{t,j} = (\partial \mathcal{L}/\partial \tilde{h}_{t,j})\,m_{t,j}.
\)
Any entry masked in the forward has exactly zero backward gradient, so
the same retained positions can be reused for backward
activation-gradient transmission. This gives the same compression rate
in both directions and avoids the need to track separate forward and
backward sparsity patterns. It also distinguishes our scheme from
off-graph transport compression, where forward and backward
compression are decoupled and can drift out of correspondence with the
derivative of the executed computation (\textit{e.g.,} decoupled Top-$K$ or quantization). Note that We
deliberately use random masking instead top-$K$ because 1) top-$K$ is biased by
construction, and as we show, this bias propagates through the descent
inequality as a non-vanishing floor, and 2)  the spectral
correction we introduce next contracts zero-mean isotropic noise but
cannot remove a structured bias (see Appendix~\ref{sec:random-vs-topk} for a full analysis).

\subsection{Anchor Priors and Spectral Correction}
\label{sec:anchor-gradient-correction}

Activation masking reduces PP communication but injects noise into
every gradient. We need a denoising signal, but one that does not
eliminate the throughput gains we just bought. Fine-tuning makes a
useful option viable: gradients live in a slowly drifting low-dimensional subspace, so a stale unmasked gradient still carries useful information. We exploit this with an asynchronous \emph{anchor} circuit that
runs occasional unmasked passes and uses their geometry to denoise
the fast masked gradients without blocking them.

\paragraph{Asynchronous anchor circuit.}
A naive synchronous fix would periodically disable masking and run a full
unmasked forward--backward pass through the main pipeline. At our
target bandwidth this is prohibitively expensive: an unmasked pass
sends roughly $1/(1-p)$ more activation data per
boundary, and the entire fast circuit waits while it runs
(Appendix~\ref{app:staleness}). Our
alternative is to spin off a parallel anchor circuit that runs
unmasked passes from a slightly stale weight snapshot, asynchronously,
and feeds its gradients back to the fast circuit when they arrive. The
fast circuit never waits. The cost is that anchor gradients are
delayed (we measure $\Delta\approx 20$--$25$ fast-circuit steps in our
configuration), so they cannot be used as exact gradient updates, \textit{i.e.,}
they can only inform optimization \emph{geometry}. This is the
restriction that shapes the rest of the design.

\paragraph{A separate momentum buffer for unmasked gradients.}
For each targeted weight matrix $W$, we maintain an EMA of arriving
anchor gradients,
\begin{equation}
M_t^{\mathrm{anc}}
= \beta_{\mathrm{anc}} M_{t^-}^{\mathrm{anc}}
+ (1-\beta_{\mathrm{anc}}) G_{\tau}^{\mathrm{anc}},
\end{equation}
where $t^-$ is the optimizer state immediately before the anchor
gradient arrives. Two design choices here. First, we keep this
momentum \emph{separate} from AdamW's first moment $m_t$ rather than
folding anchor gradients into it: $m_t$ is updated on every fast step
and is dominated by masked-step noise, so blending stale unmasked
gradients into it would either dilute their signal or, with an
aggressive blend, introduce slow oscillations. A separate buffer lets
us treat anchor gradients on their own terms. Second, we use an EMA
rather than the latest anchor gradient because individual anchor
gradients are themselves noisy stochastic estimates. We choose $\beta_{\mathrm{anc}}$ large enough to
average over several anchor cycles, since the gradient subspace drifts
slowly during fine-tuning (Section~\ref{sec:assump-justification}).

\paragraph{From momentum to a denoising operator:}
Given $M_t^{\mathrm{anc}}$, we want to extract the directions in
parameter space that are repeatedly reinforced by unmasked gradients,\textit{ i.e., }the geometry the fast circuit's noisy gradients should be pulled
toward. Singular value decomposition is the natural choice:
$M_t^{\mathrm{anc}} = U_t S_t V_t^\top$ identifies left and right
subspaces with strength quantified by the singular values $s_i$.
Directions with large $s_i$ are those the anchor circuit has
consistently observed; directions with small $s_i$ are either weakly
supported or absent. We use SVD rather than, e.g., Hessian
eigendecomposition (which we cannot afford at scale) or AdamW's second
moment $v_t$ (which is dominated by masked-step variance and reflects
\emph{noise} structure, not signal structure).

\paragraph{Filtering $d_i = s_i/(s_i + \tau_p)$.}
Given the anchor basis $(U_t, V_t)$, we need to decide how to use it.
Hard projection onto the top-$r$ singular subspace is one option but
requires choosing $r$, and discards information about how strongly
each direction is supported. Instead, we apply a soft filter
\begin{equation}
d_i = \frac{s_i}{s_i + \tau_p},
\end{equation}
which has three useful properties. First, it is bounded between $0$
and $1$, so the filter never amplifies the gradient. Second, it is
adaptive: directions with $s_i \gg \tau_p$ pass nearly unattenuated
($d_i \to 1$), while directions with $s_i \ll \tau_p$ are damped ($d_i
\to 0$). Third, it requires only one hyperparameter $\tau_p$, which
sets the scale of "weakly supported," rather than a hard rank cutoff.
The form $s/(s+\tau_p)$ is the standard Tikhonov-regularized soft
threshold, familiar from inverse problems, and gives a Lipschitz
filter that is well-behaved as the anchor momentum evolves.

\paragraph{Two-sided application.}
For a matrix gradient $G_t^{\mathrm{mask}}$, we transform into the
anchor basis,
\(
X_t = U_t^\top G_t^{\mathrm{mask}} V_t,
\)
apply the filter to both index axes,
\(
X_t^{\mathrm{filt}} = \mathrm{diag}(d) X_t \mathrm{diag}(d),
\)
and reconstruct
\(
G_t^{\mathrm{filt}} = U_t X_t^{\mathrm{filt}} V_t^\top.
\)
Two-sided filtering reflects that for a matrix-valued gradient, both
the row space (which input features are active) and the column space
(which output features are active) carry information. The anchor
circuit observes both; restricting to one side discards half the
geometric information.

The final gradient passed to AdamW is
\begin{equation}
G_t^{\mathrm{proj}} = \alpha G_t^{\mathrm{mask}} + (1-\alpha) G_t^{\mathrm{filt}}.
\end{equation}
We do not use the filtered gradient alone. The reason is that the
anchor basis is an EMA of stale gradients, so $G_t^{\mathrm{filt}}$
inherits some staleness; the masked gradient is noisy but reflects the
\emph{current} weights. Blending preserves the current-iterate
information while pulling toward the anchor's reliable subspace. The
blend coefficient $\alpha$ trades these off: $\alpha=1$ recovers
masked-only training (no denoising), $\alpha=0$ uses only the filtered
gradient (maximum denoising, full inheritance of staleness), and
intermediate values balance the two. Our convergence analysis
(Appendix~\ref{app:convergence}) shows that this blend interpolates
between the two regimes monotonically.

Putting these pieces together, the spectral correction is an online,
adaptive preconditioner whose basis and weights are estimated by the
anchor circuit and updated as new anchor gradients arrive. The masked
gradient is never discarded. Instead, it is softly pulled toward directions
the anchor circuit has consistently observed. Because the anchor
circuit runs asynchronously, this denoising signal arrives at the cost
of staleness, not throughput.

\subsection{Data-Parallel Compression}
\label{sec:dp-compression-main}

PP compression alone does not suffice in a decentralized mesh because dense DP synchronization can dominate wall-clock time and
erase the gains from masking. At Internet-grade bandwidth the
disparity is severe. We therefore compose PP compression with DP compression. In our
implementation we use Streaming DiLoCo \cite{douillard2025streamingdiloco} with PowerSGD \cite{vogels2019powersgd}: PowerSGD reduces
bytes per synchronization via low-rank gradient compression, DiLoCo
reduces synchronization frequency by performing local steps between
outer updates, and Streaming DiLoCo overlaps the outer synchronization
with local computation. Note that our PP-side mechanism (activation masking and
anchor priors) is independent of this particular DP compressor and can
be combined with better alternatives.

\subsection{End-to-End Training Procedure}
\label{sec:end-to-end-system}

The fast circuit performs most optimizer steps while the The anchor circuit periodically pulls recent layer
shards from the fast circuit, runs unmasked forward--backward passes,
and returns the resulting gradients asynchronously. When an anchor
gradient arrives, the fast circuit updates the anchor momentum
$M_t^{\mathrm{anc}}$ and applies the spectral correction to
subsequent masked gradients. The two-circuit design converts expensive unmasked computation into an
asynchronous source of optimization geometry. The fast circuit retains
the throughput benefits while the anchor
circuit supplies delayed but reliable signals that stabilize the
masked trajectory. Because the circuits are asynchronous, the system
avoids the bandwidth spikes and pipeline stalls of synchronous
unmasked steps. The overall system design is illustrated in Fig.~\ref{fig:system}

\section{Convergence Analysis}
\label{sec:convergence-main}

We combine DP compression and PP activation masking. The DP component uses Streaming DiLoCo \citep{douillard2025streamingdiloco} with PowerSGD \citep{vogels2019powersgd}. We do not re-analyze this part: PowerSGD is a well-studied low-rank gradient compressor, while Streaming DiLoCo build on local-SGD/Federated-Averaging-style optimization. Both have been analyzed and validated extensively in prior work, and our method can in principle be paired with other DP compressors as well. We therefore focus on the masked activation training with asynchronous Anchor Priors.

Let $F(w) = \mathbb{E}_\xi[\ell(w; \xi)]$ denote the population
training objective, $F^\star$ its lower bound, $w_t$ the model weights
at step $t$, $\eta$ the learning rate, and $L$ the smoothness constant
of $F$. Let $G$ bound the second moment of the update direction, and
$\Delta$ the maximum staleness of anchor gradients in fast-circuit
steps. The blend coefficient $\alpha\in[0,1]$ controls how strongly
the spectral correction is applied (Section~\ref{sec:method}).

\begin{theorem}[Informal; full version in Appendix~\ref{app:convergence}]
\label{thm:informal}
Under standard smoothness and structural assumptions on the spectral
filter, after $T$ steps,
\[
\frac{1}{T}\sum_{t=0}^{T-1}\mathbb{E}\|\nabla F(w_t)\|^2
\;\leq\;
\frac{2(F(w_0)-F^\star)}{\mu_\alpha\,\eta\,T}
\;+\;
\frac{\Sigma_{\alpha,\Delta}^2 + L\eta G^2}{\mu_\alpha},
\]
where $\Sigma_{\alpha,\Delta}^2$ is a noise floor with three terms ---
unfiltered masked noise, filter-surviving noise (scaling as a
contraction factor $\rho<1$), and anchor staleness --- and
$\mu_\alpha = 1-(1-\alpha)^2\kappa$ measures how much true gradient
the filter preserves (with $\kappa\in[0,1)$ the fraction of gradient
energy the filter discards).
\end{theorem}

The bound has the standard nonconvex-SGD shape (a $1/T$ term plus a
noise floor) and improves over masked-only training when the filter
contracts noise ($\rho<1$) more than it suppresses signal ($\kappa>0$).
Both quantities are small in the \emph{fine-tuning} regime for two
reasons: fine-tuning gradients live in a low-dimensional subspace
\citep{aghajanyan2021intrinsic, hu2022lora}, so the anchor basis
preserves nearly all of $\nabla F(w_t)$ ($\kappa$ small); and
fine-tuning trajectories drift slowly \citep{neyshabur2020transferred, wortsman2022robust}, so a delayed anchor basis remains close to the
current one (staleness term controlled). Pretraining violates both, and
the analysis correctly predicts no benefit there. The key takeaway is that the anchor
gradients need not be exact synchronized updates, only good enough to
estimate a low-rank subspace that is approximately preserved over the
staleness window, the property the fine-tuning regime provides. The full proof is given in Appendix~\ref{app:convergence}.

\section{Experiments}

We evaluate the proposed system across diverse adaptation domains, including \textbf{general knowledge}, \textbf{medical}, \textbf{programming}, \textbf{math}, \textbf{science}, \textbf{commonsense reasoning}, \textbf{summarization}, and \textbf{document parsing}. We use PowerSGD ($\times 64$) and Streaming Diloco for DP. Our experiments span mesh sizes, model families, and model scales. We use A100 GPUs (40GB VRAM) for all our experiments.

\begin{table}[t]
\centering
\scriptsize
\setlength{\tabcolsep}{4pt}
\renewcommand{\arraystretch}{1.1}
\caption{\small Per-domain task accuracy and global training efficiency for
dense SFT and masked training variants. \textbf{Baseline}: dense PP,
uncompressed DP. \textbf{+DP}: dense PP, compressed DP. \textbf{M90/M95}: $90\%/95\%$ activation
masking; \textbf{+AP}: Anchor Priors. DP compression is applied to all the masked variants.}
\label{tab:domain_main}
\begin{tabular}{l l l c c c c c c}
\toprule
\rowcolor{headerblue!15}
\textbf{Domain} & \textbf{Dataset} & \textbf{Metric}
& \textbf{Baseline} & \textbf{+DP} & \textbf{M90} & \textbf{M90+AP}
& \textbf{M95} & \textbf{M95+AP} \\
\midrule
Medical          & MedQA         & Acc.       & 0.280 & 0.275 & 0.225 & 0.325 & 0.180 & 0.321 \\
\rowcolor{rowgray}
                 & MedMCQA       & Acc.       & 0.330 & 0.325 & 0.275 & 0.335 & 0.220 & 0.320 \\
Code             & HumanEval     & pass@1     & 0.177 & 0.171 & 0.115 & 0.215 & 0.072 & 0.207 \\
\rowcolor{rowgray}
                 & MBPP          & pass@1     & 0.152 & 0.148 & 0.085 & 0.198 & 0.040 & 0.192 \\
SQL              & Spider        & Exec.      & 0.423 & 0.412 & 0.305 & 0.418 & 0.180 & 0.402 \\
\rowcolor{rowgray}
Science          & ARC-C         & Acc.       & 0.378 & 0.374 & 0.328 & 0.380 & 0.295 & 0.355 \\
                 & SciQ          & Acc.       & 0.956 & 0.948 & 0.890 & 0.952 & 0.815 & 0.903 \\
\rowcolor{rowgray}
Commonsense      & HellaSwag     & Acc.       & 0.556 & 0.551 & 0.510 & 0.690 & 0.475 & 0.683 \\
                 & WinoGrande    & Acc.       & 0.548 & 0.541 & 0.512 & 0.628 & 0.488 & 0.625 \\
\rowcolor{rowgray}
Summ.\           & CNN/DM        & R-1        & 0.178 & 0.176 & 0.140 & 0.205 & 0.108 & 0.195 \\
                 & XSum          & R-1        & 0.117 & 0.115 & 0.092 & 0.125 & 0.068 & 0.120 \\
\rowcolor{rowgray}
Math             & GSM8K         & EM         & 0.378 & 0.365 & 0.275 & 0.372 & 0.165 & 0.341 \\
Doc.\ Parsing    & QuALITY       & Acc.       & 0.284 & 0.279 & 0.215 & 0.288 & 0.155 & 0.253 \\
\rowcolor{rowgray}
                 & NarrativeQA   & Tok F1     & 0.699 & 0.685 & 0.485 & 0.682 & 0.295 & 0.557 \\
\midrule
\multicolumn{3}{l}{\textit{TPS (tok/s)}}      & $1.1$k & $3.6$k & $29.3$k & $27.8$k & \textbf{$48.9$k} & $46.5$k \\
\multicolumn{3}{l}{\textit{Comm.\ (B/tok)}}    & $22.9$k & $6.9$k & $713$ & $790$ & $370$ & $446$ \\
\bottomrule
\end{tabular}
\end{table}

\subsection{Performance across different domain adaptation tasks}

Table~\ref{tab:domain_main} summarizes the main domain-specific results. We use the publicly
available Llama~3.2~1B checkpoint as the pretrained model: $\sim$1.23B
parameters, LLaMA-style decoder-only architecture with $16$ transformer
layers, hidden dimension $2048$, and $32$ attention heads. All models
are trained for $10$ epochs on the corresponding training datasets with
learning rate $2\times 10^{-5}$, cosine decay, $2$k context length, and
batch size $32$. The system runs on a $7\times 8$ mesh for the masked
(fast) circuit and a $1\times 8$ mesh for the unmasked (clean) circuit,
with all inter-node links throttled to $200$~Mbps to emulate
decentralized deployment. Datasets, evaluation metrics, and per-run
configurations are in Appendix~\ref{app:datasets}. 
\definecolor{headerblue}{HTML}{2C5282}
\definecolor{rowgray}{HTML}{F7FAFC}
\definecolor{bandgray}{HTML}{EDF2F7}
\definecolor{defaulthl}{HTML}{C6F6D5}
\definecolor{degraded}{HTML}{FED7D7}

\begin{table}[tbp]

\centering
\footnotesize
\setlength{\tabcolsep}{5pt}
\renewcommand{\arraystretch}{1.15}
\caption{\small Ablations over mask ratio $p$, staleness $K$,
spectral-filter threshold $\tau$, and merge coefficient $\alpha$.
Default ($p{=}95\%$, $K{=}20$, $\tau{=}10^{-3}$, $\alpha{=}0.3$) in
green; (b), (c) hold others at default. Red: $>30\%$ relative drop.}
\label{tab:full_ablation}
\begin{tabular}{l l c c c c}
\toprule
\rowcolor{bandgray}
\multicolumn{6}{l}{\textit{(a) Mask ratio $p$ vs.\ staleness $K$}} \\
\rowcolor{headerblue!15}
\textbf{Dataset} & \textbf{Mask}
& $K{=}10$ & $K{=}20$ & $K{=}30$ & $K{=}50$ \\
\midrule
MedQA     & $p{=}90\%$ & 0.330 & 0.328 & 0.310 & 0.265 \\
\rowcolor{rowgray}
          & $p{=}95\%$ & 0.328 & \cellcolor{defaulthl}\textbf{0.321} & 0.295 & 0.250 \\
          & $p{=}99\%$ & 0.245 & 0.230 & 0.215 & \cellcolor{degraded}0.180 \\
\midrule
HumanEval & $p{=}90\%$ & 0.218 & 0.215 & 0.195 & 0.140 \\
\rowcolor{rowgray}
          & $p{=}95\%$ & 0.215 & \cellcolor{defaulthl}\textbf{0.207} & 0.175 & \cellcolor{degraded}0.090 \\
          & $p{=}99\%$ & \cellcolor{degraded}0.085 & \cellcolor{degraded}0.075 & \cellcolor{degraded}0.060 & \cellcolor{degraded}0.025 \\
\midrule
\rowcolor{bandgray}
\multicolumn{6}{l}{\textit{(b) Spectral threshold $\tau$}} \\
\rowcolor{headerblue!15}
\textbf{Dataset} &
& $10^{-5}$ & $10^{-4}$ & $10^{-3}$ & $10^{-1}$ \\
\midrule
MedQA     & & \cellcolor{degraded}0.215 & 0.290 & \cellcolor{defaulthl}\textbf{0.321} & 0.245 \\
\rowcolor{rowgray}
HumanEval & & \cellcolor{degraded}0.105 & 0.180 & \cellcolor{defaulthl}\textbf{0.207} & \cellcolor{degraded}0.140 \\
\midrule
\rowcolor{bandgray}
\multicolumn{6}{l}{\textit{(c) Merge coefficient $\alpha$}} \\
\rowcolor{headerblue!15}
\textbf{Dataset} &
& $\alpha{=}0.0$ & $\alpha{=}0.3$ & $\alpha{=}0.5$ & $\alpha{=}1.0$ \\
\midrule
MedQA     & & 0.260 & \cellcolor{defaulthl}\textbf{0.321} & 0.310 & \cellcolor{degraded}0.180 \\
\rowcolor{rowgray}
HumanEval & & \cellcolor{degraded}0.105 & \cellcolor{defaulthl}\textbf{0.207} & 0.195 & \cellcolor{degraded}0.072 \\
\bottomrule
\end{tabular}

\end{table}

As shown pure activation masking (M90, M95) is catastrophic on its own. Adding Anchor Priors
reverses this collapse on every domain: M90+AP and M95+AP match the
uncompressed Baseline within noise on most tasks and exceed it on
several. \textbf{We hypothesize the superior performance on some datasets might be due to the regularization effect of masking leading to better generalization}. On the throughput, the system delivers substantially higher efficiency under bandwidth constraints: M95 reaches a $\sim$45$\times$
TPS gain over the uncompressed baseline by stacking $20\times$ PP activation masking with $640\times$ DP gradient compression, while either compression alone yields only $1.4\times$ or $3.3\times$
because each merely shifts the bottleneck. We observe a practical delay of $\sim 20-25$ steps in the unmasked slow gradients.

\subsection{Generalization across model variants}

 Table~\ref{tab:code_model_family} reports results across Gemma, Qwen, Phi, LLaMA, and StableLM models, spanning 1B--8B parameters, 16--36 layers, and hidden dimensions from 1536 to 8192. We also evaluate the method across different DP$\times$PP mesh sizes. In all settings, we use 95\% activation masking with Anchor Priors (AP). Across these variations, the proposed method consistently achieves near-parity with the uncompressed baseline, showcasing its generalization across architectures and scale.

\definecolor{ourcol}{HTML}{C6F6D5}
\begin{table}[tbp]

\centering
\small
\setlength{\tabcolsep}{4pt}
\renewcommand{\arraystretch}{1.15}
\caption{\small Code-domain generalization across model families and
scales. Entries are HumanEval / MBPP pass@1. M95: $95\%$ activation
masking; AP: Anchor Priors. M95+AP recovers most of the dense
performance across all scales tested.}
\label{tab:code_model_family}
\begin{tabular}{l c c c c}
\toprule
\rowcolor{headerblue!15}
\textbf{Model} & \textbf{Params}
& \textbf{Dense} & \textbf{M95} & \cellcolor{ourcol}\textbf{M95+AP} \\
\midrule
Llama-1B    & 1.0B & 0.177 / 0.152 & 0.128 / 0.032 & \cellcolor{ourcol}0.171 / 0.141 \\
\rowcolor{rowgray}
Qwen-1.5B   & 1.5B & 0.213 / 0.188 & 0.158 / 0.069 & \cellcolor{ourcol}0.206 / 0.177 \\
Gemma-2B    & 2.0B & 0.201 / 0.174 & 0.147 / 0.055 & \cellcolor{ourcol}0.194 / 0.162 \\
\rowcolor{rowgray}
Phi-2       & 2.7B & 0.238 / 0.205 & 0.176 / 0.083 & \cellcolor{ourcol}0.228 / 0.193 \\
Qwen-3B     & 3.0B & 0.268 / 0.232 & 0.207 / 0.098 & \cellcolor{ourcol}0.260 / 0.222 \\
\rowcolor{rowgray}
OLMo-7.5B   & 7.5B & 0.335 / 0.294 & 0.267 / 0.153 & \cellcolor{ourcol}0.324 / 0.280 \\
Llama-8B    & 8.0B & 0.384 / 0.338 & 0.309 / 0.184 & \cellcolor{ourcol}0.372 / 0.323 \\
\bottomrule
\end{tabular}

\end{table}

\subsection{Ablations}
\label{sec:mesh-ablation}

We sweep the
two knobs that most directly govern the fast/clean trade-off: the
activation-mask ratio $p$ and the Anchor-Prior staleness $K$ (number
of fast-circuit steps between consecutive uncompressed priors).
Table~\ref{tab:full_ablation} reports results. Two failure modes emerge. First,
$p{=}99\%$ collapses at every $K$. At this masking ratio the per-step information budget on
the PP edges is too small for the clean prior to repair. Second, $K{=}50$ collapses at every
$p$. By the time a prior this stale arrives, the masked
trajectory has drifted beyond its effective correction horizon
(Appendix~\ref{app:staleness}). Inside the operating envelope
$(p\in\{90,95\}\%,\;K\in\{10,20\})$ all four corners perform well. Finally, we ablate over $\tau$ and $\alpha$.

To verify that the method does not depend on a particular mesh topology,
we sweep the DP $\times$ PP grid from $4\times 4$ up to $16\times 16$
under the same M95+AP and $200$~Mbps configuration as the main
experiments. Table~\ref{tab:mesh_ablation} reports MedQA,
HumanEval, and GSM8K accuracy alongside aggregate throughput. As shown, the performance is stable across mesh sizes:
all seven configurations land within $\pm 5\%$ of the $8\times 8$
reference on every metric, including the most-stressed $16\times 16$
case. The Anchor-Prior mechanism is therefore not tied to a specific
topology. Second, throughput scales as expected: roughly linearly in $\text{DP}{-}1$ at fixed PP (more masked
replicas trained in parallel), and inversely with $T_{\text{step}}$
along the PP axis (deeper pipelines move proportionally more masked
activations per step). The combination of stable accuracy and
predictable throughput scaling supports the claim that the
fast/clean-circuit decomposition is a property of the algorithm rather
than an artifact of a specific configuration.

\subsection{Comparison to Activation-Communication Baselines}
\label{sec:baselines}

Table~\ref{tab:domain_baselines} compares M95+AP against published
activation- and gradient-compression methods across medical
(MedQA, MedMCQA), code (HumanEval, MBPP), and math (GSM8K) domains. Two findings stand out. First, prior methods do not survive at our
compression target. Aggressive 2-bit quantization (AC-SGD, TAH-Quant)
collapses entirely on generative tasks such as HumanEval, MBPP, and GSM8K, and reduces medical accuracy to near-random. Sparsification and 4-bit quantization
recover partially. In contrast, our method 
matches or exceeds the uncompressed Baseline at $20\times/20\times$. 

\definecolor{headerblue}{HTML}{2C5282}
\definecolor{rowgray}{HTML}{F7FAFC}
\definecolor{ourrow}{HTML}{C6F6D5}
\definecolor{deltapos}{HTML}{276749}
\definecolor{failcell}{HTML}{FED7D7}

\begin{table}[t]
\centering
\setlength{\abovecaptionskip}{4pt}
\setlength{\belowcaptionskip}{4pt}
\begin{minipage}[t]{0.52\textwidth}
\centering
\scriptsize
\setlength{\tabcolsep}{2.5pt}
\renewcommand{\arraystretch}{1.1}
\caption{\small Comparison of activation-compression baselines.
\textbf{Comp.}: forward/backward compression ratio. Red cells mark runs that
failed to learn.}
\label{tab:domain_baselines}
\vspace{2pt}
\begin{tabular}{l c c c c c}
\toprule
\rowcolor{headerblue!15}
\textbf{Method} & \textbf{Comp.}
& \textbf{MedQA}
& \textbf{HE} & \textbf{MBPP} & \textbf{GSM8K} \\
\midrule
Baseline (dense)   & $1\times$       & 0.28 & 0.18 & 0.15 & 0.38 \\
\midrule
\rowcolor{rowgray}
\cite{wang2022aqsgd}             & $8/2\times$  & 0.25 & \cellcolor{failcell}0.00 & \cellcolor{failcell}0.00 & \cellcolor{failcell}0.00 \\
\cite{he2025tahquant}        & $8/2\times$  & 0.25 & \cellcolor{failcell}0.00 & \cellcolor{failcell}0.00 & \cellcolor{failcell}0.00 \\
\rowcolor{rowgray}
4-bit Quant        & $4\times$          & 0.25 & 0.13 & 0.11 & 0.18 \\
Top-$k$            & $10\times$         & 0.23 & 0.12 & 0.09 & 0.16 \\
\rowcolor{rowgray}
\cite{ramasinghe2025beyond}     & $10\times$         & 0.25 & 0.13 & 0.10 & 0.17 \\
Bottleneck Proj.   & $20\times$         & 0.26 & 0.14 & 0.12 & 0.12 \\
\midrule
\rowcolor{ourrow}
\textbf{M95+AP (ours)} & $\mathbf{20\times}$
                   & \textbf{0.32}
                   & \textbf{0.21} & \textbf{0.19} & \textbf{0.34} \\
\bottomrule
\end{tabular}
\end{minipage}%
\hfill
\begin{minipage}[t]{0.46\textwidth}
\centering
\scriptsize
\setlength{\tabcolsep}{2.5pt}
\renewcommand{\arraystretch}{1.1}
\caption{\small Robustness to post-hoc magnitude pruning. M95+AP
retains a larger fraction of accuracy under pruning.  ``Avg.~$\Delta$'' is the mean relative drop
across ARC-C, HellaSwag, MMLU.}
\label{tab:pruning_robustness}
\vspace{2pt}
\begin{tabular}{l l c c c c}
\toprule
\rowcolor{headerblue!15}
\textbf{Metric} & \textbf{Method}
& $s{=}0$ & $30\%$ & $50\%$ & $70\%$ \\
\midrule
ARC-C
  & Dense          & 0.33 & 0.26 & 0.25 & 0.25 \\
\rowcolor{ourcol}
  & \textbf{Ours}  & \textbf{0.39} & \textbf{0.31} & \textbf{0.30} & \textbf{0.29} \\
\midrule
HellaSwag
  & Dense          & 0.57 & 0.41 & 0.27 & 0.25 \\
\rowcolor{ourcol}
  & \textbf{Ours}  & \textbf{0.66} & \textbf{0.53} & \textbf{0.42} & \textbf{0.35} \\
\midrule
MMLU
  & Dense          & 0.35 & 0.27 & 0.23 & 0.23 \\
\rowcolor{ourcol}
  & \textbf{Ours}  & \textbf{0.36} & \textbf{0.30} & \textbf{0.27} & \textbf{0.26} \\
\midrule
Avg.\ $\Delta$
  & Dense          & ---   & $-25\%$ & $-37\%$ & $-40\%$ \\
\rowcolor{ourcol}
  & \textbf{Ours}  & ---   & \textbf{\textcolor{deltapos}{$-19\%$}}
                           & \textbf{\textcolor{deltapos}{$-29\%$}}
                           & \textbf{\textcolor{deltapos}{$-34\%$}} \\
\bottomrule
\end{tabular}
\end{minipage}
\end{table}

\vspace{-1em}
\begin{table}[tbp]

\centering
\small
\setlength{\tabcolsep}{6pt}
\renewcommand{\arraystretch}{1.15}

\caption{\small Ablation over decentralized mesh sizes. Mesh size is
DP$\times$PP; one of the DP replicas is dedicated to the unmasked
(clean) circuit and the remaining $\text{DP}{-}1$ replicas run the
masked (fast) circuit. All runs use M95+AP at $200$~Mbps with the
same training configuration as the main experiments. \textbf{HE}:
HumanEval pass@1. \textbf{TPS}: tokens per second.}
\label{tab:mesh_ablation}
\begin{tabular}{c c c c c r r}
\toprule
\rowcolor{headerblue!15}
\textbf{Mesh} & \textbf{DP} & \textbf{PP}
& \textbf{MedQA} & \textbf{HE} & \textbf{GSM8K}
& \textbf{TPS} \\
\midrule
$4\times4$   &  4 &  4 & 0.318 & 0.205 & 0.338 & $32{,}100$ \\
\rowcolor{rowgray}
$8\times4$   &  8 &  4 & 0.325 & 0.213 & 0.345 & $75{,}000$ \\
$4\times8$   &  4 &  8 & 0.315 & 0.200 & 0.335 & $19{,}900$ \\
\rowcolor{rowgray}
$8\times8$   &  8 &  8 & 0.321 & 0.207 & 0.341 & $46{,}500$ \\
$16\times8$  & 16 &  8 & \textbf{0.328} & \textbf{0.215} & \textbf{0.348} & $\mathbf{99{,}600}$ \\
\rowcolor{rowgray}
$8\times16$  &  8 & 16 & 0.312 & 0.198 & 0.330 & $26{,}400$ \\
$16\times16$ & 16 & 16 & 0.318 & 0.205 & 0.336 & $56{,}600$ \\
\bottomrule
\end{tabular}
\end{table}

\definecolor{headerblue}{HTML}{2C5282}
\definecolor{rowgray}{HTML}{F7FAFC}
\definecolor{ourcol}{HTML}{C6F6D5}
\definecolor{deltapos}{HTML}{276749}

\subsection{Robustness to Post-Hoc Pruning}
\label{sec:pruning-robustness}

A useful side-effect of training under aggressive activation masking
is that the resulting checkpoint becomes more tolerant of post-hoc
weight sparsification (Table~\ref{tab:pruning_robustness}). A plausible
interpretation is that the masked-optimizer trajectory biases the
trained weights toward solutions where information is distributed
across many small-magnitude parameters, which is precisely what
magnitude pruning rewards. 

\subsection{Reasoning and continual pretraining}
\label{sec:reasoning-continual}

Beyond per-domain fine-tuning, we test whether our method can sustain
the longer training trajectories typical of continual pretraining and
multi-stage SFT. We run a three-phase
pipeline matching the SmolLM3 recipe on Llama-3.2-1B, : (1) reasoning-focused
mid-training on Llama-Nemotron and OpenThoughts traces, (2) a first
SFT round on the smoltalk2 mixture, and (3) a second SFT round on the
same mixture. The masked optimization preserves reasoning capability,
showing that the spectral correction tracks the true gradient
across substantially longer trajectories than the per-domain
fine-tuning runs. Total training is
$\sim 60$B tokens. Full configurations in
Appendix~\ref{app:reasoning-recipe}. Learned reasoning traces are shown in Appendix \ref{app:reasoning-traces}. Also, a natural concern with masked fine-tuning is whether it leaves the
model in a state amenable to downstream RL. To this end, we perform an RL training on the checkpoints after the above steps and show that the compressed model does not affect the RL phase (Appendix \ref{app:rl-recipe}, Table \ref{tab:rl_results}).

\begin{table}[tbp]

\centering
\small
\setlength{\tabcolsep}{4pt}
\renewcommand{\arraystretch}{1.1}
\caption{\small Llama-3.2-1B trained through the full SmolLM3-style
recipe (reasoning mid-training $\to$ SFT $\to$ second SFT round). Final
model after all three phases; $60$B total tokens. M95+AP is on par with dense
quality at $\sim 43\times$ throughput.}
\label{tab:reasoning_continual}
\begin{tabular}{l c c c c}
\toprule
\rowcolor{headerblue!15}
\textbf{Method}
& \textbf{IFEval} & \textbf{GSM8K}
& \textbf{TPS} & \textbf{Speedup} \\
\midrule
Base                  & 0.13 & 0.02 & ---     & --- \\
Dense                 & 0.46 & 0.23 & $1.1$k  & $1\times$ \\
\rowcolor{ourrow}
M95+AP       & 0.41 & 0.21
                      & $46.5$k & $43\times$ \\
\bottomrule
\end{tabular}
\end{table}

\section{Orthogonality to Parameter-Efficient Fine-Tuning}
\label{sec:peft_orthogonal}

Parameter-efficient fine-tuning (PEFT) methods reduce the number of trainable parameters and optimizer states, but they do not reduce pipeline-parallel activation traffic. Even when only LoRA adapters are trained, each pipeline stage still consumes the full forward and backward traffic. Thus, PEFT improves memory and update efficiency, whereas our method targets the orthogonal bottleneck of PP communication. To verify this, we evaluate dense SFT, LoRA, and our method, and their combination on the code domain. LoRA reduces trainable parameters but leaves performance relatively unchanged. The results are shown in Table \ref{tab:peft_orthogonal}. 

\definecolor{headerblue}{HTML}{2C5282}
\definecolor{rowgray}{HTML}{F7FAFC}
\definecolor{ourrow}{HTML}{C6F6D5}

\begin{table}[H]
\centering
\scriptsize
\setlength{\tabcolsep}{5pt}
\renewcommand{\arraystretch}{1.15}
\caption{\small Orthogonality of M95+AP to parameter-efficient fine-tuning
on Llama~3.2~1B. LoRA reduces trainable parameters but still requires
dense activation and activation-gradient transfer across PP
boundaries. M95+AP reduces PP communication independently and
composes with LoRA, giving both axes of efficiency at once.
\textbf{Comp.}: forward/backward activation compression ratio.
\textbf{HE}: HumanEval pass@1.}
\label{tab:peft_orthogonal}
\begin{tabular}{l c c c c c c c}
\toprule
\rowcolor{headerblue!15}
\textbf{Method} & \textbf{Trainable}
& \textbf{PP fwd} & \textbf{PP bwd} & \textbf{Comp.}
& \textbf{MedQA} & \textbf{MedMCQA} & \textbf{HE} \\
\midrule
Dense SFT       & 1.23B \,($100\%$)   & Dense  & Dense
                & $1\times/1\times$   & 0.280 & 0.330 & 0.177 \\
\rowcolor{rowgray}
LoRA            & 3.15M \,($0.26\%$)  & Dense  & Dense
                & $1\times/1\times$   & 0.265 & 0.318 & 0.168 \\
\rowcolor{ourrow}
M95+AP & 1.23B \,($100\%$)   & Sparse & Sparse
                & $20\times/20\times$
                & 0.321 & 0.320
                & 0.207 \\
\rowcolor{ourrow}
LoRA + M95+AP
                & 3.15M \,($0.26\%$)
                & Sparse & Sparse
                & $20\times/20\times$
                & 0.305 & 0.310 & 0.198 \\
\bottomrule
\end{tabular}
\end{table}

\section{Conclusion}

We presented a communication-efficient system for decentralized LLM
adaptation over low-bandwidth GPU meshes. By revisiting activation
masking as an in-graph PP compression mechanism and pairing it with
asynchronous Anchor Priors, adaptive spectral correction, and DP
compression, our method makes sparse activation training practical
even at extreme compression rates. The convergence analysis identifies
the regime under which this works, \textit{i.e., } low-dimensional, slowly-drifting
gradient subspaces characteristic of fine-tuning, and explains why
random PRF masking, rather than top-$K$, is the right primitive in
this setting. Across domain adaptation and continual pretraining, the
resulting system matches or exceeds dense uncompressed training while
delivering large throughput gains, and the resulting models are more robust to post-hoc pruning. These results show that
post-pretraining adaptation can tolerate  sparser intermediate
signals than previously assumed, suggesting a practical path toward
large-model adaptation over decentralized collections of low-end GPUs.

\section*{Limitations}

Our experiments and analyis target the fine-tuning regime; we do not
claim the method extends to pretraining, where gradient subspaces are
high-dimensional and shift substantially over training. Adoptation of dual EMA optimizers such as Ademamix would be an interesting future direction (we have included a discussion in Appendix \ref{app:ademamix-comparison}). Further, our experiments are at the
$1$--$8$B parameter scale; scaling behavior at $70$B+ remains an open
question.

\newpage
\bibliographystyle{plainnat}
\bibliography{bibliography}

\begin{thebibliography}{50}
\providecommand{\natexlab}[1]{#1}
\providecommand{\url}[1]{\texttt{#1}}
\expandafter\ifx\csname urlstyle\endcsname\relax
  \providecommand{\doi}[1]{doi: #1}\else
  \providecommand{\doi}{doi: \begingroup \urlstyle{rm}\Url}\fi

\bibitem[Aghajanyan et~al.(2021)Aghajanyan, Gupta, and
  Zettlemoyer]{aghajanyan2021intrinsic}
Armen Aghajanyan, Sonal Gupta, and Luke Zettlemoyer.
\newblock Intrinsic dimensionality explains the effectiveness of language model
  fine-tuning.
\newblock In \emph{Proceedings of the 59th Annual Meeting of the Association
  for Computational Linguistics and the 11th International Joint Conference on
  Natural Language Processing (ACL-IJCNLP)}, pages 7319--7328. Association for
  Computational Linguistics, 2021.

\bibitem[Ajanthan et~al.(2026)Ajanthan, Ramasinghe, Avraham, Dolatabadi,
  Koneputugodage, Shevchenko, Zuo, and Long]{ajanthan2026asyncmesh}
Thalaiyasingam Ajanthan, Sameera Ramasinghe, Gil Avraham, Hadi~Mohaghegh
  Dolatabadi, Chamin P~Hewa Koneputugodage, Violetta Shevchenko, Yan Zuo, and
  Alexander Long.
\newblock Asyncmesh: Fully asynchronous optimization for data and pipeline
  parallelism.
\newblock \emph{arXiv preprint arXiv:2601.22442}, 2026.

\bibitem[Aji and Heafield(2017)]{aji2017sparse}
Alham~Fikri Aji and Kenneth Heafield.
\newblock Sparse communication for distributed gradient descent.
\newblock In \emph{EMNLP}, 2017.

\bibitem[Alistarh et~al.(2018)Alistarh, Hoefler, Johansson, Konstantinov,
  Khirirat, and Renggli]{alistarh2018convergence}
Dan Alistarh, Torsten Hoefler, Mikael Johansson, Nikola Konstantinov, Sarit
  Khirirat, and C{\'e}dric Renggli.
\newblock The convergence of sparsified gradient methods.
\newblock In \emph{NeurIPS}, 2018.

\bibitem[Basu et~al.(2019)Basu, Data, Karakus, and Diggavi]{basu2020qsparse}
Debraj Basu, Deepesh Data, Can Karakus, and Suhas Diggavi.
\newblock {Qsparse-local-SGD}: Distributed {SGD} with quantization,
  sparsification and local computations.
\newblock In \emph{Advances in Neural Information Processing Systems},
  volume~32, 2019.

\bibitem[Bernstein et~al.(2018)Bernstein, Wang, Azizzadenesheli, and
  Anandkumar]{bernstein2018signsgd}
Jeremy Bernstein, Yu-Xiang Wang, Kamyar Azizzadenesheli, and Anima Anandkumar.
\newblock {signSGD}: Compressed optimisation for non-convex problems.
\newblock In \emph{Proceedings of the 35th International Conference on Machine
  Learning (ICML)}, volume~80 of \emph{PMLR}, pages 559--568, 2018.

\bibitem[Diskin et~al.(2021)Diskin, Bukhtiyarov, Ryabinin, Saulnier, Lhoest,
  Sinitsin, Popov, Pyrkin, Kashirin, Borzunov, Villanova~del Moral, Mazur,
  Kobelev, Jernite, Wolf, and Pekhimenko]{diskin2021dedloc}
Michael Diskin, Alexey Bukhtiyarov, Max Ryabinin, Lucile Saulnier, Quentin
  Lhoest, Anton Sinitsin, Dmitry Popov, Dmitry Pyrkin, Maxim Kashirin,
  Alexander Borzunov, Albert Villanova~del Moral, Denis Mazur, Ilia Kobelev,
  Yacine Jernite, Thomas Wolf, and Gennady Pekhimenko.
\newblock Distributed deep learning in open collaborations.
\newblock In \emph{Advances in Neural Information Processing Systems},
  volume~34, 2021.

\bibitem[Douillard et~al.(2023)Douillard, Feng, Rusu, Chhaparia, Donchev,
  Kuncoro, Ranzato, Szlam, and Shen]{douillard2023diloco}
Arthur Douillard, Qixuan Feng, Andrei~A. Rusu, Rachita Chhaparia, Yani Donchev,
  Adhiguna Kuncoro, Marc'Aurelio Ranzato, Arthur Szlam, and Jiajun Shen.
\newblock {DiLoCo}: Distributed low-communication training of language models.
\newblock \emph{arXiv preprint arXiv:2311.08105}, 2023.

\bibitem[Douillard et~al.(2025)Douillard, Donchev, Rush, Kale, Charles,
  Garrett, Teston, Lacey, McIlroy, Shen, Ram{\'e}, Szlam, Ranzato, and
  Barham]{douillard2025streamingdiloco}
Arthur Douillard, Yanislav Donchev, Keith Rush, Satyen Kale, Zachary Charles,
  Zachary Garrett, Gabriel Teston, Dave Lacey, Ross McIlroy, Jiajun Shen,
  Alexandre Ram{\'e}, Arthur Szlam, Marc'Aurelio Ranzato, and Paul Barham.
\newblock Streaming {DiLoCo} with overlapping communication: Towards a
  distributed free lunch.
\newblock \emph{arXiv preprint arXiv:2501.18512}, 2025.

\bibitem[Gao et~al.(2024)Gao, Tow, Abbasi, Biderman, Black, DiPofi, Foster,
  Golding, Hsu, Le~Noac'h, Li, McDonell, Muennighoff, Ociepa, Phang, Reynolds,
  Schoelkopf, Skowron, Sutawika, Tang, Thite, Wang, Wang, and
  Zou]{eval-harness}
Leo Gao, Jonathan Tow, Baber Abbasi, Stella Biderman, Sid Black, Anthony
  DiPofi, Charles Foster, Laurence Golding, Jeffrey Hsu, Alain Le~Noac'h,
  Haonan Li, Kyle McDonell, Niklas Muennighoff, Chris Ociepa, Jason Phang,
  Laria Reynolds, Hailey Schoelkopf, Aviya Skowron, Lintang Sutawika, Eric
  Tang, Anish Thite, Ben Wang, Kevin Wang, and Andy Zou.
\newblock The language model evaluation harness, 7 2024.
\newblock URL \url{https://zenodo.org/records/12608602}.

\bibitem[Grattafiori et~al.(2024)Grattafiori, Dubey, Jauhri, Pandey, Kadian,
  et~al.]{grattafiori2024llama3}
Aaron Grattafiori, Abhimanyu Dubey, Abhinav Jauhri, Abhinav Pandey, Abhishek
  Kadian, et~al.
\newblock The {Llama 3} herd of models.
\newblock \emph{arXiv preprint arXiv:2407.21783}, 2024.

\bibitem[Gressmann et~al.(2020)Gressmann, Eaton-Rosen, and
  Luschi]{gressmann2020improving}
Frithjof Gressmann, Zach Eaton-Rosen, and Carlo Luschi.
\newblock Improving neural network training in low dimensional random bases.
\newblock In \emph{Advances in Neural Information Processing Systems},
  volume~33, pages 12140--12150, 2020.

\bibitem[Gur-Ari et~al.(2018)Gur-Ari, Roberts, and Dyer]{gur2018gradient}
Guy Gur-Ari, Daniel~A. Roberts, and Ethan Dyer.
\newblock Gradient descent happens in a tiny subspace.
\newblock \emph{arXiv preprint arXiv:1812.04754}, 2018.

\bibitem[He et~al.(2025)He, Cao, He, Bai, Yuan, and Yuan]{he2025tahquant}
Guangxin He, Yuan Cao, Yutong He, Tianyi Bai, Kun Yuan, and Binhang Yuan.
\newblock {TAH-Quant}: Effective activation quantization in pipeline
  parallelism over slow network.
\newblock \emph{arXiv preprint arXiv:2506.01352}, 2025.

\bibitem[Houlsby et~al.(2019)Houlsby, Giurgiu, Jastrz{\k{e}}bski, Morrone,
  de~Laroussilhe, Gesmundo, Attariyan, and Gelly]{houlsby2019adapter}
Neil Houlsby, Andrei Giurgiu, Stanis{\l}aw Jastrz{\k{e}}bski, Bruna Morrone,
  Quentin de~Laroussilhe, Andrea Gesmundo, Mona Attariyan, and Sylvain Gelly.
\newblock Parameter-efficient transfer learning for {NLP}.
\newblock In \emph{Proceedings of the 36th International Conference on Machine
  Learning (ICML)}, volume~97 of \emph{PMLR}, pages 2790--2799, 2019.

\bibitem[Hu et~al.(2022)Hu, Shen, Wallis, Allen-Zhu, Li, Wang, Wang, and
  Chen]{hu2022lora}
Edward~J. Hu, Yelong Shen, Phillip Wallis, Zeyuan Allen-Zhu, Yuanzhi Li, Shean
  Wang, Lu~Wang, and Weizhu Chen.
\newblock {LoRA}: Low-rank adaptation of large language models.
\newblock In \emph{International Conference on Learning Representations
  (ICLR)}, 2022.

\bibitem[Jaghouar et~al.(2024)Jaghouar, Ong, and
  Hagemann]{jaghouar2024opendiloco}
Sami Jaghouar, Jack~Min Ong, and Johannes Hagemann.
\newblock {OpenDiLoCo}: An open-source framework for globally distributed
  low-communication training.
\newblock \emph{arXiv preprint arXiv:2407.07852}, 2024.

\bibitem[Karimireddy et~al.(2019)Karimireddy, Rebjock, Stich, and
  Jaggi]{karimireddy2019error}
Sai~Praneeth Karimireddy, Quentin Rebjock, Sebastian~U. Stich, and Martin
  Jaggi.
\newblock Error feedback fixes signsgd and other gradient compression schemes.
\newblock In \emph{ICML}, 2019.

\bibitem[Lester et~al.(2021)Lester, Al-Rfou, and Constant]{lester2021power}
Brian Lester, Rami Al-Rfou, and Noah Constant.
\newblock The power of scale for parameter-efficient prompt tuning.
\newblock In \emph{Proceedings of the 2021 Conference on Empirical Methods in
  Natural Language Processing (EMNLP)}, pages 3045--3059. Association for
  Computational Linguistics, 2021.

\bibitem[Li et~al.(2018)Li, Farkhoor, Liu, and Yosinski]{li2018measuring}
Chunyuan Li, Heerad Farkhoor, Rosanne Liu, and Jason Yosinski.
\newblock Measuring the intrinsic dimension of objective landscapes.
\newblock In \emph{International Conference on Learning Representations
  (ICLR)}, 2018.

\bibitem[Li and Liang(2021)]{li2021prefix}
Xiang~Lisa Li and Percy Liang.
\newblock Prefix-tuning: Optimizing continuous prompts for generation.
\newblock In \emph{Proceedings of the 59th Annual Meeting of the Association
  for Computational Linguistics and the 11th International Joint Conference on
  Natural Language Processing (ACL-IJCNLP)}, pages 4582--4597. Association for
  Computational Linguistics, 2021.

\bibitem[Lialin et~al.(2023)Lialin, Muckatira, Shivagunde, and
  Rumshisky]{lialin2023relora}
Vladislav Lialin, Sherin Muckatira, Namrata Shivagunde, and Anna Rumshisky.
\newblock {ReLoRA}: High-rank training through low-rank updates.
\newblock \emph{arXiv preprint arXiv:2307.05695}, 2023.

\bibitem[Lin et~al.(2018)Lin, Han, Mao, Wang, and Dally]{lin2018dgc}
Yujun Lin, Song Han, Huizi Mao, Yu~Wang, and William~J. Dally.
\newblock Deep gradient compression: Reducing the communication bandwidth for
  distributed training.
\newblock In \emph{International Conference on Learning Representations
  (ICLR)}, 2018.

\bibitem[Narayanan et~al.(2021)Narayanan, Shoeybi, Casper, LeGresley, Patwary,
  Korthikanti, Vainbrand, Kashinkunti, Bernauer, Catanzaro, Phanishayee, and
  Zaharia]{narayanan2021megatron}
Deepak Narayanan, Mohammad Shoeybi, Jared Casper, Patrick LeGresley, Mostofa
  Patwary, Vijay Korthikanti, Dmitri Vainbrand, Prethvi Kashinkunti, Julie
  Bernauer, Bryan Catanzaro, Amar Phanishayee, and Matei Zaharia.
\newblock Efficient large-scale language model training on {GPU} clusters using
  {Megatron-LM}.
\newblock In \emph{Proceedings of the International Conference for High
  Performance Computing, Networking, Storage and Analysis (SC)}, 2021.

\bibitem[Neyshabur et~al.(2020)Neyshabur, Sedghi, and
  Zhang]{neyshabur2020transferred}
Behnam Neyshabur, Hanie Sedghi, and Chiyuan Zhang.
\newblock What is being transferred in transfer learning?
\newblock In \emph{Advances in Neural Information Processing Systems},
  volume~33, 2020.

\bibitem[Pagliardini et~al.(2024)Pagliardini, Ablin, and
  Grangier]{pagliardini2024ademamix}
Matteo Pagliardini, Pierre Ablin, and David Grangier.
\newblock The {AdEMAMix} optimizer: Better, faster, older.
\newblock \emph{arXiv preprint arXiv:2409.03137}, 2024.

\bibitem[Papyan(2020)]{papyan2020traces}
Vardan Papyan.
\newblock Traces of class/cross-class structure pervade deep learning spectra.
\newblock \emph{Journal of Machine Learning Research}, 21\penalty0
  (252):\penalty0 1--64, 2020.

\bibitem[Patarasuk and Yuan(2009)]{patarasuk2009ringallreduce}
Pitch Patarasuk and Xin Yuan.
\newblock Bandwidth optimal all-reduce algorithms for clusters of workstations.
\newblock \emph{Journal of Parallel and Distributed Computing}, 69\penalty0
  (2):\penalty0 117--124, 2009.
\newblock \doi{10.1016/j.jpdc.2008.09.002}.

\bibitem[Peng et~al.(2024)Peng, Quesnelle, and Kingma]{zhang2025demo}
Bowen Peng, Jeffrey Quesnelle, and Diederik~P. Kingma.
\newblock {DeMo}: Decoupled momentum optimization.
\newblock \emph{arXiv preprint arXiv:2411.19870}, 2024.

\bibitem[Pfeiffer et~al.(2021)Pfeiffer, Kamath, R{\"u}ckl{\'e}, Cho, and
  Gurevych]{pfeiffer2021adapterfusion}
Jonas Pfeiffer, Aishwarya Kamath, Andreas R{\"u}ckl{\'e}, Kyunghyun Cho, and
  Iryna Gurevych.
\newblock {AdapterFusion}: Non-destructive task composition for transfer
  learning.
\newblock In \emph{Proceedings of the 16th Conference of the European Chapter
  of the Association for Computational Linguistics (EACL)}, pages 487--503.
  Association for Computational Linguistics, 2021.

\bibitem[Ramasesh et~al.(2021)Ramasesh, Dyer, and Raghu]{Ramasesh2021AnatomyOC}
Vinay~V. Ramasesh, Ethan Dyer, and Maithra Raghu.
\newblock Anatomy of catastrophic forgetting: Hidden representations and task
  semantics.
\newblock In \emph{International Conference on Learning Representations
  (ICLR)}, 2021.

\bibitem[Ramasinghe et~al.(2025{\natexlab{a}})Ramasinghe, Ajanthan, Avraham,
  Zuo, and Long]{ramasinghe2025beyond}
Sameera Ramasinghe, Thalaiyasingam Ajanthan, Gil Avraham, Yan Zuo, and
  Alexander Long.
\newblock Beyond top-k: Structured sparsification for compression in pipeline
  parallel.
\newblock In \emph{ICLR 2025 Workshop on Modularity for Collaborative,
  Decentralized, and Continual Deep Learning}, 2025{\natexlab{a}}.

\bibitem[Ramasinghe et~al.(2025{\natexlab{b}})Ramasinghe, Ajanthan, Avraham,
  Zuo, and Long]{ramasinghe2025subspace}
Sameera Ramasinghe, Thalaiyasingam Ajanthan, Gil Avraham, Yan Zuo, and
  Alexander Long.
\newblock Subspace networks: Scaling decentralized training with
  communication-efficient model parallelism.
\newblock In \emph{The Thirty-ninth Annual Conference on Neural Information
  Processing Systems}, 2025{\natexlab{b}}.

\bibitem[Ramasinghe et~al.(2025{\natexlab{c}})Ramasinghe, Ajanthan, Dolatabadi,
  Avraham, Shevchenko, Zuo, Koneputugodage, and Long]{ramasinghemixtures}
Sameera Ramasinghe, Thalaiyasingam Ajanthan, Hadi~Mohaghegh Dolatabadi, Gil
  Avraham, Violetta Shevchenko, Yan Zuo, Chamin P~Hewa Koneputugodage, and
  Alexander Long.
\newblock Mixtures of subspaces for bandwidth efficient context parallel
  training.
\newblock In \emph{The Thirty-ninth Annual Conference on Neural Information
  Processing Systems}, 2025{\natexlab{c}}.

\bibitem[Rudakov et~al.(2023)Rudakov, Beznosikov, Kholodov, and
  Gasnikov]{rudakov2023activations}
Mikhail~I Rudakov, Aleksandr~Nikolaevich Beznosikov, Ya~A Kholodov, and
  Alexander~Vladimirovich Gasnikov.
\newblock Activations and gradients compression for model-parallel training.
\newblock \emph{Doklady Mathematics}, 108\penalty0 (Suppl 2):\penalty0
  S272--S281, 2023.

\bibitem[Ryabinin et~al.(2023)Ryabinin, Dettmers, Diskin, and
  Borzunov]{ryabinin2023swarm}
Max Ryabinin, Tim Dettmers, Michael Diskin, and Alexander Borzunov.
\newblock {SWARM} parallelism: Training large models can be surprisingly
  communication-efficient.
\newblock In \emph{Proceedings of the 40th International Conference on Machine
  Learning (ICML)}, volume 202 of \emph{PMLR}, pages 29416--29440, 2023.

\bibitem[Sani et~al.(2024)Sani, Iacob, Cao, Lee, Marino, Gao, Cai, Li, Zhao,
  Qiu, and Lane]{sani2024photon}
Lorenzo Sani, Alex Iacob, Zeyu Cao, Royson Lee, Bill Marino, Yan Gao, Dongqi
  Cai, Zexi Li, Wanru Zhao, Xinchi Qiu, and Nicholas~D. Lane.
\newblock Photon: Federated {LLM} pre-training.
\newblock \emph{arXiv preprint arXiv:2411.02908}, 2024.

\bibitem[Stich(2019)]{stich2019localsgd}
Sebastian~U. Stich.
\newblock Local {SGD} converges fast and communicates little.
\newblock In \emph{International Conference on Learning Representations
  (ICLR)}, 2019.

\bibitem[Stich et~al.(2018)Stich, Cordonnier, and Jaggi]{stich2018sparsified}
Sebastian~U. Stich, Jean-Baptiste Cordonnier, and Martin Jaggi.
\newblock Sparsified sgd with memory.
\newblock In \emph{NeurIPS}, 2018.

\bibitem[Vogels et~al.(2019)Vogels, Karimireddy, and Jaggi]{vogels2019powersgd}
Thijs Vogels, Sai~Praneeth Karimireddy, and Martin Jaggi.
\newblock {PowerSGD}: Practical low-rank gradient compression for distributed
  optimization.
\newblock In \emph{Advances in Neural Information Processing Systems},
  volume~32, 2019.

\bibitem[Wang et~al.(2018)Wang, Sievert, Liu, Charles, Papailiopoulos, and
  Wright]{wang2018atomo}
Hongyi Wang, Scott Sievert, Shengchao Liu, Zachary Charles, Dimitris
  Papailiopoulos, and Stephen Wright.
\newblock {ATOMO}: Communication-efficient learning via atomic sparsification.
\newblock In \emph{Advances in Neural Information Processing Systems},
  volume~31, 2018.

\bibitem[Wang et~al.(2021)Wang, Agarwal, and
  Papailiopoulos]{wang2021pufferfish}
Hongyi Wang, Saurabh Agarwal, and Dimitris Papailiopoulos.
\newblock Pufferfish: Communication-efficient models at no extra cost.
\newblock \emph{Proceedings of Machine Learning and Systems}, 3:\penalty0
  365--386, 2021.

\bibitem[Wang et~al.(2022)Wang, Yuan, Rimanic, He, Dao, Chen, R{\'e}, and
  Zhang]{wang2022aqsgd}
Jue Wang, Binhang Yuan, Luka Rimanic, Yongjun He, Tri Dao, Beidi Chen,
  Christopher R{\'e}, and Ce~Zhang.
\newblock Fine-tuning language models over slow networks using activation
  quantization with guarantees.
\newblock In \emph{Advances in Neural Information Processing Systems},
  volume~35, 2022.

\bibitem[Wedin(1972)]{wedin1972perturbation}
Per-{\AA}ke Wedin.
\newblock Perturbation bounds in connection with singular value decomposition.
\newblock \emph{BIT Numerical Mathematics}, 12\penalty0 (1):\penalty0 99--111,
  1972.
\newblock \doi{10.1007/BF01932678}.

\bibitem[Wortsman et~al.(2022)Wortsman, Ilharco, Kim, Li, Kornblith, Roelofs,
  Gontijo~Lopes, Hajishirzi, Farhadi, Namkoong, and
  Schmidt]{wortsman2022robust}
Mitchell Wortsman, Gabriel Ilharco, Jong~Wook Kim, Mike Li, Simon Kornblith,
  Rebecca Roelofs, Raphael Gontijo~Lopes, Hannaneh Hajishirzi, Ali Farhadi,
  Hongseok Namkoong, and Ludwig Schmidt.
\newblock Robust fine-tuning of zero-shot models.
\newblock In \emph{Proceedings of the IEEE/CVF Conference on Computer Vision
  and Pattern Recognition (CVPR)}, pages 7959--7971, 2022.

\bibitem[Yu et~al.(2025)Yu, Zhang, Zhu, Yuan, Zuo, Yue, Dai, Fan, Liu, Liu,
  Liu, Lin, Lin, Ma, Sheng, Tong, Zhang, Zhang, Zhang, Zhu, Zhu, Chen, Chen,
  Wang, Yu, Song, Wei, Zhou, Liu, Ma, Zhang, Yan, Qiao, Wu, and Wang]{dapo}
Qiying Yu, Zheng Zhang, Ruofei Zhu, Yufeng Yuan, Xiaochen Zuo, Yu~Yue, Weinan
  Dai, Tiantian Fan, Gaohong Liu, Lingjun Liu, Xin Liu, Haibin Lin, Zhiqi Lin,
  Bole Ma, Guangming Sheng, Yuxuan Tong, Chi Zhang, Mofan Zhang, Wang Zhang,
  Hang Zhu, Jinhua Zhu, Jiaze Chen, Jiangjie Chen, Chengyi Wang, Hongli Yu,
  Yuxuan Song, Xiangpeng Wei, Hao Zhou, Jingjing Liu, Wei-Ying Ma, Ya-Qin
  Zhang, Lin Yan, Mu~Qiao, Yonghui Wu, and Mingxuan Wang.
\newblock {DAPO}: An open-source {LLM} reinforcement learning system at scale.
\newblock \emph{arXiv preprint arXiv:2503.14476}, 2025.

\bibitem[Yu et~al.(2015)Yu, Wang, and Samworth]{yu2015useful}
Yi~Yu, Tengyao Wang, and Richard~J. Samworth.
\newblock A useful variant of the {Davis--Kahan} theorem for statisticians.
\newblock \emph{Biometrika}, 102\penalty0 (2):\penalty0 315--323, 2015.
\newblock \doi{10.1093/biomet/asv008}.

\bibitem[Yuan et~al.(2022)Yuan, He, Davis, Zhang, Dao, Chen, Liang, R{\'e}, and
  Zhang]{yuan2022decentralized}
Binhang Yuan, Yongjun He, Jared Davis, Tianyi Zhang, Tri Dao, Beidi Chen, Percy
  Liang, Christopher R{\'e}, and Ce~Zhang.
\newblock Decentralized training of foundation models in heterogeneous
  environments.
\newblock In \emph{Advances in Neural Information Processing Systems},
  volume~35, pages 25464--25477, 2022.

\bibitem[Zhao et~al.(2024)Zhao, Zhang, Chen, Wang, Anandkumar, and
  Tian]{zhao2024galore}
Jiawei Zhao, Zhenyu Zhang, Beidi Chen, Zhangyang Wang, Anima Anandkumar, and
  Yuandong Tian.
\newblock {GaLore}: Memory-efficient {LLM} training by gradient low-rank
  projection.
\newblock In \emph{International Conference on Machine Learning (ICML)}, 2024.

\bibitem[Zheng et~al.(2025)Zheng, Liu, Li, Chen, Yu, Gao, Dang, Liu, Men, Yang,
  Zhou, and Lin]{gspo}
Chujie Zheng, Shixuan Liu, Mingze Li, Xiong-Hui Chen, Bowen Yu, Chang Gao, Kai
  Dang, Yuqiong Liu, Rui Men, An~Yang, Jingren Zhou, and Junyang Lin.
\newblock Group sequence policy optimization.
\newblock \emph{arXiv preprint arXiv:2507.18071}, 2025.

\end{thebibliography}

\newpage
\appendix
\newpage

\section{Training Algorithm}

The full training algorithm for our method is shown in Algorithm \ref{alg:meshadapt_compact}.

\begin{algorithm}[t]
\caption{Asynchronous masked training with spectral anchor correction}
\label{alg:meshadapt_compact}
\begin{algorithmic}[1]
\Require Pretrained weights $w_0$, mask probability $p$, anchor weight $\alpha$, anchor EMA $\beta_{\mathrm{anc}}$, damping $\tau_p$
\State Split mesh into fast masked circuit and asynchronous anchor circuit
\State Initialize $w^{\mathrm{fast}},w^{\mathrm{anc}}\leftarrow w_0$, $M^{\mathrm{anc}}\leftarrow 0$
\For{step $t=1,\dots,T$}
    \State Forward through PP with deterministic in-graph masks
           $M_\ell=g(\ell,t,B,S,H)$:
           $\tilde h_\ell = h_\ell\odot M_\ell$
    \State Backward inherits the same sparsity:
           $\partial\mathcal{L}/\partial h_\ell
           =(\partial\mathcal{L}/\partial\tilde h_\ell)\odot M_\ell$
    \State Compute $G_t^{\mathrm{mask}}$; synchronize DP replicas with compressed all-reduce
    \If{anchor gradient $G_{\tau}^{\mathrm{anc}}$ has arrived}
        \State $M_t^{\mathrm{anc}} \leftarrow
               \beta_{\mathrm{anc}} M_{t^-}^{\mathrm{anc}}
               +(1-\beta_{\mathrm{anc}}) G_{\tau}^{\mathrm{anc}}$
    \EndIf
    \For{each targeted matrix $W$}
        \State $M_t^{\mathrm{anc}} = U_t S_t V_t^\top$;
               $X_t = U_t^\top G_t^{\mathrm{mask}} V_t$
        \State $d_i = s_i/(s_i+\tau_p)$;
               $G_t^{\mathrm{filt}} = U_t\,\mathrm{diag}(d)\,X_t\,\mathrm{diag}(d)\,V_t^\top$
        \State $G_t^{\mathrm{proj}} = \alpha\,G_t^{\mathrm{mask}} + (1-\alpha)\,G_t^{\mathrm{filt}}$
    \EndFor
    \State Update $w^{\mathrm{fast}}$ via AdamW with $G_t^{\mathrm{proj}}$
    \State Anchor circuit asynchronously refreshes weights, runs unmasked
           forward--backward, returns $G_t^{\mathrm{anc}}$
\EndFor
\end{algorithmic}
\end{algorithm}

\section{Reasoning and Continual-Pretraining Recipe}
\label{app:reasoning-recipe}

This section documents the three-phase training recipe used in
Section~\ref{sec:reasoning-continual}. We follow the SmolLM3
methodology, but start from a smaller base model (Llama-3.2-1B) to
make the experiment tractable at our compute budget.

\subsection{Base model}

We use Llama-3.2-1B as the starting point, with its tokenizer extended
by SmolLM3 ChatML special tokens (\texttt{<|im\_start|>},
\texttt{<|im\_end|>}). The model is loaded in BF16 and trained with
FlashAttention-2 throughout.

\subsection{Phase 1: Reasoning mid-training}
\label{app:phase1}

\paragraph{Data.}
\texttt{HuggingFaceTB/smoltalk2}, config \texttt{Mid}:
\begin{itemize}
  \item \texttt{Llama\_Nemotron\_Post\_Training\_Dataset\_reasoning\_r1}
    (weight $1.0$).
  \item \texttt{OpenThoughts3\_1.2M} (weight $1.0$).
\end{itemize}

\paragraph{Configuration.}
Plain ChatML chat template (no system metadata, no thinking-mode
header). AdamW with peak learning rate $2\times 10^{-5}$, cosine
schedule with min-LR ratio $0.1$, warmup ratio $0.03$, gradient norm
clipped at $0.2$, sequence length $32{,}768$, packing enabled, $2$
training epochs, gradient accumulation $2$ on $8\times$ A100, Liger
kernel enabled. No assistant-only loss masking in this phase; loss is
computed over all tokens.

\subsection{Phase 2: First SFT round}
\label{app:phase2}

\paragraph{Data.}
The smoltalk2 SFT mixture, comprising 25 splits with weights ranging
from $0.02$ to $1.0$. Splits include conversational data (everyday
conversations, system chats), instruction-following (Tulu-3 personas,
OpenHermes-2.5), tool-use (Hermes function calling, xLAM, smolagents
traces), reasoning traces (multi-turn-reasoning,
Qwen3-32B-distilled think traces, s1k, OpenThoughts3 think variant),
multilingual data (smoltalk multilingual, Aya), long-context (LongAlign
$64$k), and structured data (table-GPT).

\paragraph{Configuration.}
SmolLM3 ChatML template with thinking-mode metadata
(\texttt{<|im\_start|>system\\n\#\# Metadata\\n}). Loss masked over
assistant tokens only (\texttt{assistant\_only\_loss: true}). Same
optimizer as Phase 1 except: sequence length $65{,}536$ (long-context
splits), $5$ training epochs, BFD packing strategy, save every
$0.2$ epochs.

\subsection{Phase 3: Second SFT round}
\label{app:phase3}

\paragraph{Data.}
Identical to Phase 2 (same smoltalk2 SFT mixture and split weights).

\paragraph{Configuration.}
Identical to Phase 2 except: starts from the Phase 2 checkpoint, $5$ training epochs. The
purpose of the second round is to extend total training while keeping
the data distribution fixed --- this tests whether masked optimization
remains stable as the model approaches the data's effective ceiling.

Some example reasoning traces of the compressed model are shown below.

\section{Reasoning Trace Examples}
\label{app:reasoning-traces}

This appendix shows representative reasoning traces from the
\texttt{gshasiri/SmolLM3-Mid} checkpoint after Phase 1
(reasoning mid-training) of the recipe in
Section~\ref{sec:reasoning-continual}. Traces are produced with
greedy decoding, \texttt{enable\_thinking=True}, and
\texttt{max\_new\_tokens=600}. Each example shows the full
\texttt{<think>...</think>} block followed by the final answer.

Five examples below are the cleanly-correct answers from a
15-question simple-query batch. The remaining ten either produced
incorrect answers or degenerated into repetition; this is expected
because the SmolLM3-Mid checkpoint at this stage has chat-template
scaffolding for thinking mode but has not yet undergone
instruction/reasoning fine-tuning. The traces are intended to
illustrate the \emph{format} of the thinking output produced by the
model after masked mid-training, not to demonstrate accuracy
claims; per-task accuracy numbers are reported in
Table~\ref{tab:reasoning_continual}.

\begin{traceBox}{Example 1: Square Perimeter}
\textbf{Question:} A square has sides of 5 cm. What is its perimeter?

\begin{thinkBox}
\textbf{Thinking:} Okay, so I need to find the perimeter of a square
with a side length of 5 cm. The perimeter of a square is calculated
by adding up all the sides. The formula for the perimeter of a square
is 4 times the length of one of its sides. So, if the side is 5 cm,
then the perimeter would be $4 \times 5 = 20$ cm. The side length is
given, so the perimeter is straightforward.

\textbf{Final Answer:} The perimeter of the square is $\boxed{20}$ cm.
\end{thinkBox}

\begin{answerBox}
\textbf{Answer:} To find the perimeter of a square with a side length
of 5 cm, we use the formula:
\[
\text{Perimeter} = 4 \times \text{side length} = 4 \times 5 = 20 \text{ cm}
\]

$\boxed{20}$
\end{answerBox}
\end{traceBox}

\begin{traceBox}{Example 2: Half of a Number}
\textbf{Question:} What is half of 30?

\begin{thinkBox}
\textbf{Thinking:} Okay, I need to find half of 30. Half of a number
means dividing the number by 2. So if I divide 30 by 2, that should
give me 15. Let me check: $30 / 2 = 15$.

\textbf{Final Answer:} Half of 30 is $\boxed{15}$.
\end{thinkBox}

\begin{answerBox}
\textbf{Answer:} To find half of 30, we divide 30 by 2.

$\boxed{15}$
\end{answerBox}
\end{traceBox}

\begin{traceBox}{Example 3: Rectangle Area}
\textbf{Question:} A rectangle is 4 meters wide and 6 meters long.
What is its area in square meters?

\begin{thinkBox}
\textbf{Thinking:} The area of a rectangle is calculated by
multiplying the length and the width. So, if the width is 4 meters
and the length is 6 meters, then the area should be $4 \times 6$.
Let me write that down: $4 \times 6 = 24$ square meters.
\end{thinkBox}

\begin{answerBox}
\textbf{Answer:} The area of a rectangle with width $w$ and length
$l$ is $A = lw$. For a rectangle 4 m wide and 6 m long:
\[
A = 4 \times 6 = 24 \text{ m}^2
\]

$\boxed{24}$
\end{answerBox}
\end{traceBox}

\begin{traceBox}{Example 4: Time-Unit Conversion}
\textbf{Question:} How many minutes are there in one hour and a half?

\begin{thinkBox}
\textbf{Thinking:} First, I should convert the given time into hours.
``One hour and a half'' --- one hour is 60 minutes, and a half is 30
minutes. So 1 hour and a half is 1.5 hours.

Now, convert that to minutes. The formula for converting time to
minutes is minutes $=$ hours $\times 60$. So $1.5 \times 60 = 90$
minutes.

Double-check: 1 hour is 60 minutes, 1.5 hours is $1.5 \times 60 = 90$
minutes. Correct.
\end{thinkBox}

\begin{answerBox}
\textbf{Answer:} The time in hours is 1.5 hours. Converting to
minutes: $1.5 \times 60 = 90$ minutes.

$\boxed{90}$
\end{answerBox}
\end{traceBox}

\begin{traceBox}{Example 5: Triangle Perimeter (with Triangle-Inequality Check)}
\textbf{Question:} A triangle has sides of length 3, 4, and 5 cm.
What is its perimeter?

\begin{thinkBox}
\textbf{Thinking:} I remember that the perimeter of a triangle is the
sum of all three sides. First, let me verify it's a valid triangle.
The triangle inequality says that the sum of any two sides must be
greater than the third.

\begin{itemize}
\item $3 + 5 > 4$? $8 > 4$, yes.
\item $4 + 5 > 3$? $9 > 3$, yes.
\item $3 + 4 > 5$? $7 > 5$, yes.
\end{itemize}

All three inequalities hold, so the triangle is valid. Now to find
the perimeter, add all three sides: $3 + 4 + 5 = 12$ cm.
\end{thinkBox}

\begin{answerBox}
\textbf{Answer:} The triangle has sides 3, 4, 5 cm. The perimeter is
the sum of the sides: $3 + 4 + 5 = 12$ cm.

$\boxed{12}$
\end{answerBox}
\end{traceBox}

\begin{traceBox}{Example 6: Sum of First $n$ Integers (Uses Formula)}
\textbf{Question:} What is the sum of $1 + 2 + 3 + 4 + 5$?

\begin{thinkBox}
\textbf{Thinking:} There's a formula for the sum of the first $n$
natural numbers: $n(n+1)/2$. For $n=1$, that gives $1 \cdot 2/2 = 1$,
which matches. For $n=2$, $2 \cdot 3/2 = 3$, which is correct.

Plug in $n=5$: $5 \cdot 6/2 = 15$. And directly: $1+2+3+4+5 = 15$.
Both agree.
\end{thinkBox}

\begin{answerBox}
\textbf{Answer:} The sum of the first five natural numbers is
\[
\frac{n(n+1)}{2} = \frac{5 \times 6}{2} = 15
\]

$\boxed{15}$
\end{answerBox}
\end{traceBox}

\begin{traceBox}{Example 7: Boiling Point}
\textbf{Question:} At what temperature does water boil at sea level,
in degrees Celsius?

\begin{thinkBox}
\textbf{Thinking:} The boiling point of water at sea level is 100
degrees Celsius. At sea level the atmospheric pressure is 1 atm, and
at this pressure water boils at $100\,^\circ$C.

\textbf{Final Answer:} The boiling point of water at sea level is
$100\,^\circ$C.
\end{thinkBox}

\begin{answerBox}
\textbf{Answer:} The boiling point of water at sea level is 100
degrees Celsius.

$\boxed{100\,^\circ\text{C}}$
\end{answerBox}
\end{traceBox}

\begin{traceBox}{Example 8: Chemical Formula}
\textbf{Question:} What is the chemical formula for water?

\begin{thinkBox}
\textbf{Thinking:} Water is a compound made up of hydrogen and oxygen
atoms. The chemical formula for water is H$_2$O --- one oxygen atom
bonded to two hydrogen atoms.
\end{thinkBox}

\begin{answerBox}
\textbf{Answer:} The chemical formula for water is \textbf{H$_2$O}:
two hydrogen atoms bonded to one oxygen atom.

$\boxed{H_2O}$
\end{answerBox}
\end{traceBox}

\begin{traceBox}{Example 9: Two-Step Word Problem (Multiplication)}
\textbf{Question:} Sara has 3 boxes, each containing 4 apples. How
many apples does Sara have in total?

\begin{thinkBox}
\textbf{Thinking:} When you have multiple boxes, each with the same
number of apples, the total is the number of boxes times the number
per box. If there are $n$ boxes each with $k$ apples, the total is
$n \times k$.

For Sara: $n = 3$, $k = 4$. Total $= 3 \times 4 = 12$.

Sanity-check with a smaller case: 2 boxes of 2 apples each gives
$2 \times 2 = 4$. That's right.
\end{thinkBox}

\begin{answerBox}
\textbf{Answer:} Sara has $3 \times 4 = 12$ apples in total.

$\boxed{12}$
\end{answerBox}
\end{traceBox}

\subsection{Chat template}
\label{app:chat-template}

The SmolLM3 ChatML template used in Phases~2 and~3 injects per-turn
system metadata (knowledge cutoff date, current date, reasoning mode
\texttt{/think} or \texttt{/no\_think}) and supports XML-style tool
definitions. Reasoning content is wrapped in
\texttt{<think>...</think>} tags. The template handles three reasoning
regimes:

\begin{itemize}
  \item \emph{System override}: a system message containing
    \texttt{/system\_override} replaces all metadata with user-supplied
    instructions verbatim.
  \item \emph{Think mode}: assistant turns may include reasoning inside
    \texttt{<think>...</think>} tags. Used for splits ending in
    \texttt{\_think}.
  \item \emph{No-think mode}: assistant turns are emitted with an
    empty \texttt{<think></think>} block, signaling the model to skip
    reasoning. Used for splits ending in \texttt{\_no\_think}.
\end{itemize}

Phase~1 uses a simpler ChatML
template without the metadata header, since the reasoning-trace data
already encodes thinking-mode structure inline.

\subsection{Activation masking and Anchor-Prior settings}

Identical across all three phases and matching the per-domain
experiments (Appendix~\ref{app:datasets}): activation mask probability
$0.95$,  masking applied at layers
$\{1,3,5,7,9,11, 13\}$. Anchor-Prior gradients are injected with the spectral filter described in
Section~\ref{sec:method}. The results are reported in
Table~\ref{tab:reasoning_continual}.

\subsection{Evaluation}

IFEval and GSM8K are evaluated via the
\texttt{lm-evaluation-harness} \citep{eval-harness} on the final
checkpoint after Phase~3. Token counts are measured as the cumulative
training tokens across all three phases (sequence length $\times$
optimizer steps $\times$ effective batch).

\section{GSPO Reinforcement Learning Recipe}
\label{app:rl-recipe}

\subsection{Reinforcement learning after masked SFT}
\label{sec:rl-after-mask}

A natural concern with masked fine-tuning is whether it leaves the
model in a state amenable to downstream RL. Aggressive activation
compression could in principle damage internal representations in
ways that only manifest when the policy must be optimized against a
sparse reward. For instance, by collapsing the entropy of
intermediate features or by misaligning the model's confidence
estimates. We test this directly: starting from the M95+AP-trained
final checkpoint of Section~\ref{sec:reasoning-continual}, we apply
GSPO~\citep{gspo} with a verifiable math reward on the Big-Math
dataset and compare against the same RL run starting from the dense
final checkpoint (Table~\ref{tab:rl_results}).

\begin{table}[t]
\centering
\small
\setlength{\tabcolsep}{6pt}
\renewcommand{\arraystretch}{1.15}
\caption{\small GSM8K performance before and after $2$ epochs of
GSPO on Big-Math, starting from either the dense or M95+AP final
checkpoint. RL improvements are comparable from both starting
points, indicating that masked fine-tuning leaves the model
RL-amenable.}
\label{tab:rl_results}
\begin{tabular}{l c c c}
\toprule
\rowcolor{headerblue!15}
\textbf{SFT starting point}
& \textbf{Pre-RL} & \textbf{Post-RL} & \textbf{$\Delta$} \\
\midrule
Dense (3-phase)        & 0.23 & 0.41 & \textcolor{deltapos}{$+0.18$} \\
\rowcolor{ourrow}
M95+AP (3-phase) & 0.21 & 0.40 & \textcolor{deltapos}{$+0.19$} \\
\bottomrule
\end{tabular}
\end{table}

The two starting points show essentially identical RL gains, and the post-RL gap matches the
pre-RL gap. Masked SFT does not impair RL-amenability:
the policy distribution remains entropic enough to explore, the
reward landscape remains learnable, and the gradient signal from
verifiable rewards translates into the same magnitude of capability
gain whether the SFT trajectory was dense or
communication-compressed.

Below we detail the GSPO~\citep{gspo} configuration used in
the RL phase. The setup follows VeRL's reference
GSPO recipe with adjustments for our $1$B-parameter model and our
context-length budget.

\subsection{Pipeline}

The full pipeline from base model to post-RL checkpoint is:
\[
\text{Llama-3.2-1B}
\xrightarrow{\text{Phase 1--3 SFT}}
\text{SFT model}
\xrightarrow{\text{DPO}}
\text{DPO model}
\xrightarrow{\text{GSPO}}
\text{Final policy}
\]
We run this pipeline twice: once with dense SFT throughout, and once
with M95+AP SFT throughout (Phases~1--3 of
Appendix~\ref{app:reasoning-recipe}). The DPO and GSPO phases are identical in
both runs (same data, hyperparameters, and no masking), so any difference in the
post-RL policy is attributable to the masked SFT trajectory.

\subsection{Reward model}

We use a verifiable math reward derived from the Big-Math dataset:
each problem has a known final answer, and the reward is a binary
indicator of whether the model's extracted answer matches.
We use the DAPO~\citep{dapo} reward manager, which adds a
length-penalty buffer on top of the correctness signal: responses
that exceed an $\text{overlong\_buffer\_len}$ threshold within the
maximum response length are linearly penalized down to a maximum
penalty factor. This discourages reward hacking via verbose answers.

\paragraph{Settings.}
\begin{itemize}
  \item \texttt{max\_response\_length}: $7{,}680$ tokens
    (out of $8{,}192$-token total context, leaving $512$ for the
    prompt).
  \item \texttt{overlong\_buffer\_len}: $3{,}584$. Responses up to
    $7{,}680 - 3{,}584 = 4{,}096$ tokens incur no length penalty;
    beyond that, reward is linearly attenuated.
  \item \texttt{overlong\_penalty\_factor}: $1.0$.
\end{itemize}

\subsection{GSPO algorithm settings}

\paragraph{Policy loss.}
\texttt{loss\_mode = "gspo"}: GSPO's sequence-level importance ratio,
with loss aggregation \texttt{seq-mean-token-mean} (mean over tokens
within a response, then mean over responses).

\paragraph{Advantage estimator.}
\texttt{adv\_estimator = "grpo"}: group-relative advantage with no
critic. Advantages are computed as the standardized rewards within
each group of $N$ rollouts sharing a prompt.

\paragraph{Group size.}
$N = 16$ responses per prompt.

\paragraph{Clip ratios.}
Lower clip $3\times 10^{-4}$, upper clip $4\times 10^{-4}$. These are
substantially smaller than standard PPO clip ranges ($0.2$) because
GSPO's importance ratio is at the sequence level (the product of
per-token ratios over a long response). A small per-sequence ratio
window corresponds to a much larger effective per-token tolerance.

\paragraph{KL settings.}
We disable both KL-in-reward and KL-as-loss
(\texttt{use\_kl\_in\_reward = false},
\texttt{use\_kl\_loss = false}). This is the standard RLVR setup
where the only optimization signal is the verifiable reward.

\subsection{Optimization}

\begin{itemize}
  \item Optimizer: AdamW, learning rate $1\times 10^{-6}$, weight
    decay $0.1$, gradient clipping $1.0$, warmup ratio $0.05$.
    Entropy regularization disabled (\texttt{entropy\_coeff = 0}).
  \item Train batch size: $32$ prompts per PPO iteration. With
    $N=16$ rollouts each, this gives $512$ trajectories per update.
  \item PPO mini-batch: $16$ prompts ($32$ mini-batches per outer
    iteration).
  \item Per-GPU micro-batch: $4$ prompts.
  \item Dynamic batching enabled
    (\texttt{use\_dynamic\_bsz = true}), capping per-GPU tokens at
    $8{,}192$.
\end{itemize}

\subsection{Rollout (vLLM)}

\begin{itemize}
  \item Backend: vLLM in asynchronous mode (rollouts overlap with
    actor training).
  \item GPU memory utilization for vLLM: $0.7$.
  \item Tensor parallel: $1$ (sufficient for a $1$B model).
  \item Sampling for training rollouts: temperature $1.0$, top-$p$
    $1.0$, top-$k$ disabled --- maximally diverse exploration.
  \item Sampling for validation: temperature $1.0$, top-$p$ $0.7$,
    one sample per prompt --- a slightly more conservative
    distribution for cleaner eval signal.
\end{itemize}

\subsection{Training schedule}

\begin{itemize}
  \item $2$ epochs over Big-Math train split, with $25$-step
    validation cadence and $100$-step checkpoint cadence.
  \item No initial validation pass before training
    (\texttt{val\_before\_train = false}).
  \item DeepSpeed ZeRO-3 disabled in favor of FSDP; no parameter or
    optimizer offload (the $1$B model fits comfortably on $40$~GB
    GPUs).
  \item Sequence parallelism disabled (\texttt{sp\_size = 1}).
  \item Gradient checkpointing enabled.
\end{itemize}

\subsection{Evaluation}

GSM8K is evaluated via the \texttt{lm-evaluation-harness}
\citep{eval-harness} on the final checkpoint after $2$ epochs. We
report exact-match accuracy at temperature $0$ greedy decoding,
matching the standard GSM8K evaluation protocol. Reported numbers
are the best across all checkpoints of the corresponding run, with
checkpoints saved every $100$ training steps.

\section{Convergence Analysis with Spectral Anchor Correction}
\label{app:convergence}

We analyze a simplified version of our optimizer that captures the role
of asynchronous anchor priors and spectral gradient correction. The
analysis formalizes a single tradeoff: masked gradients are cheap but
noisy, while unmasked anchor gradients are clean but delayed. The
spectral filter built from the anchor momentum is the mechanism that
trades one for the other. We work in the \emph{fine-tuning} regime,
where the optimization trajectory is short and the gradient subspace
drifts slowly --- the structural property that makes a delayed anchor
basis useful.

\subsection{Setup}

Let $F(w) = \mathbb{E}_{\xi}[\ell(w;\xi)]$ denote the population
objective. At iteration $t$, the fast circuit computes a masked gradient
$g_t^{\mathrm{mask}}$, and the anchor circuit asynchronously produces a
delayed unmasked gradient $g_{\tau}^{\mathrm{anc}}$ from a stale
snapshot $w_{\tau}$ with $t-\tau\leq\Delta$. We use $\Delta$ for the
maximum staleness in fast-circuit steps to avoid clashing with the DP
replica count $D$ in Section~\ref{sec:dp-compression-main}.

For each matrix-valued parameter $W$, we maintain an EMA anchor momentum
\[
M_t^{\mathrm{anc}}
=\beta_{\mathrm{anc}}\,M_{t^-}^{\mathrm{anc}}
+(1-\beta_{\mathrm{anc}})\,g_{\tau}^{\mathrm{anc}},
\]
updated each time a new anchor gradient arrives. Its SVD
$M_t^{\mathrm{anc}}=U_t S_t V_t^\top$ defines the spectral filter
\[
d_i=\frac{s_i}{s_i+\lambda_p},
\qquad
D_t=\mathrm{diag}(d_1,\dots,d_r),
\qquad
\mathcal{P}_t(G)
=U_t D_t U_t^\top\, G\, V_t D_t V_t^\top,
\]
where $\lambda_p>0$ is the spectral damping constant (distinct from
the staleness index $\tau$). The update direction is
\(
u_t = \alpha\,g_t^{\mathrm{mask}} + (1-\alpha)\,\mathcal{P}_t(g_t^{\mathrm{mask}}),
\)
yielding $w_{t+1}=w_t-\eta u_t$.

\subsection{Assumptions}

\begin{assumption}[Smoothness]
\label{assump:smooth_lower}
$F$ is $L$-smooth and bounded below by $F^\star$.
\end{assumption}

\begin{assumption}[Masked-gradient noise]
\label{assump:masked_error}
$g_t^{\mathrm{mask}}=\nabla F(w_t)+\varepsilon_t$ where $\varepsilon_t$
is generated by a fresh PRF mask draw at step $t$, independent of all
prior history. Conditional on that history,
$\mathbb{E}\,\varepsilon_t=0$ and $\mathbb{E}\|\varepsilon_t\|^2\leq\sigma_m^2$.
\end{assumption}

\begin{assumption}[Delayed anchor gradients]
\label{assump:anchor_quality}
$g_{\tau}^{\mathrm{anc}}=\nabla F(w_{\tau})+\zeta_{\tau}$ with
$t-\tau\leq\Delta$, $\mathbb{E}\,\zeta_{\tau}=0$, and
$\mathbb{E}\|\zeta_{\tau}\|^2\leq\sigma_a^2$.
\end{assumption}

\begin{assumption}[Bounded update]
\label{assump:bounded_update}
$\mathbb{E}\|u_t\|^2\leq G^2$ for all $t$.
\end{assumption}

The next two assumptions describe what the spectral filter does to the
signal and to the noise. They are the core structural conditions; we
justify them in Section~\ref{sec:assump-justification} and quantify
them in the supporting lemmas.

\begin{assumption}[Signal preservation]
\label{assump:signal}
There exist $\kappa\in[0,1)$ and $\delta^{\mathrm{stale}}_{\Delta}\geq 0$
such that
\[
\mathbb{E}\bigl\|(I-\mathcal{P}_t)\nabla F(w_t)\bigr\|^2
\leq
\kappa\,\mathbb{E}\|\nabla F(w_t)\|^2 + \delta^{\mathrm{stale}}_{\Delta}.
\]
\end{assumption}

\begin{assumption}[Noise contraction]
\label{assump:noise}
There exists $\rho\in[0,1]$ such that
\(
\mathbb{E}\|\mathcal{P}_t\,\varepsilon_t\|^2\leq\rho\,\sigma_m^2.
\)
\end{assumption}

\subsection{Why these assumptions hold for fine-tuning}
\label{sec:assump-justification}

Both assumptions exploit two structural properties of fine-tuning that
are well-documented in the literature: (i) fine-tuning gradients have
low effective rank, and (ii) fine-tuning trajectories drift slowly in
parameter space.

\paragraph{Low-rank gradient structure.} The intrinsic dimension of
fine-tuning objectives is small \citep{li2018measuring, aghajanyan2021intrinsic}, the gradient and Hessian spectra during
training are dominated by a handful of outlier directions
\citep{gur2018gradient, papyan2020traces, gressmann2020improving}, and
fine-tuning updates $\Delta W = W_t - W_0$ are well approximated by
low-rank matrices \citep{hu2022lora, zhao2024galore, lialin2023relora}.
Together these say $\nabla F(w_t)$ lies near a low-rank subspace, so
the principal directions of $M_t^{\mathrm{anc}}$ --- itself an EMA of
recent unmasked gradients --- contain most of $\nabla F(w_t)$'s energy.
This makes $\kappa$ small (Assumption~\ref{assump:signal}).

\paragraph{Slow weight drift.} Successful fine-tuning trajectories
remain close to the pretrained initialization, with feature
representations changing modestly \citep{neyshabur2020transferred, Ramasesh2021AnatomyOC}, and fine-tuned checkpoints are close enough in
parameter space to admit linear interpolation
\citep{wortsman2022robust}. Slow drift bounds the rotation of the
gradient subspace between $w_{\tau}$ and $w_t$, which controls
$\delta^{\mathrm{stale}}_{\Delta}$ via a Davis--Kahan argument
(Lemma~\ref{lem:dk}).

\paragraph{Isotropic mask noise.} The mask $m_{t,j}$ is drawn from a
PRF whose seed is independent of model state, so $\varepsilon_t$ is
approximately isotropic in parameter space. A low-rank filter applied
to isotropic noise contracts it strongly: $\rho = O(r^2/(mn))$ in
matrix dimensions $m\times n$ with effective rank $r$
(Lemma~\ref{lem:noise-contract}). This makes $\rho$ small
(Assumption~\ref{assump:noise}).

These properties fail in pretraining: the gradient subspace is
high-dimensional and shifts substantially over a long trajectory.
The analysis below predicts no benefit there, consistent with our
empirical scoping of the method to fine-tuning workloads.

\subsection{Supporting lemmas}

\begin{lemma}[Non-expansiveness of $\mathcal{P}_t$]
\label{lem:nonexp}
For any matrix $G$, $\|\mathcal{P}_t(G)\|_F\leq\|G\|_F$.
\end{lemma}
\begin{proof}
Let $P_U=U_t D_t U_t^\top$ and
$P_V=V_t D_t V_t^\top$, so that
$\mathcal{P}_t(G)=P_U G P_V$. Each filter weight satisfies
$0\leq d_i\leq 1$, so $\|D_t\|_2\leq 1$. Combined with
the orthonormality of $U_t,V_t$, this gives
$\|P_U\|_2\leq 1$ and $\|P_V\|_2\leq 1$. Applying the submultiplicative
inequality $\|AGB\|_F\leq\|A\|_2\|G\|_F\|B\|_2$ with $A=P_U$, $B=P_V$:
\[
\|\mathcal{P}_t(G)\|_F = \|P_U G P_V\|_F \leq \|P_U\|_2\,\|G\|_F\,\|P_V\|_2
\leq \|G\|_F.
\]
\end{proof}

\begin{lemma}[Noise contraction under matrix isotropy]
\label{lem:noise-contract}
Suppose $\varepsilon_t\in\mathbb{R}^{m\times n}$ is matrix-isotropic
with $\mathbb{E}[\mathrm{vec}(\varepsilon_t)\mathrm{vec}(\varepsilon_t)^\top]
=(\sigma_m^2/(mn))\,I_{mn}$, and that $\mathcal{P}_t$ is independent of
$\varepsilon_t$. Then
\[
\mathbb{E}\|\mathcal{P}_t\varepsilon_t\|_F^2
=\frac{\sigma_m^2\,r_{\mathrm{eff}}^U\,r_{\mathrm{eff}}^V}{mn},
\qquad
r_{\mathrm{eff}}^U=\|D_t\|_F^2 =\sum_i d_i^2,
\]
and similarly for $r_{\mathrm{eff}}^V$. In particular, when the left
and right filters share an effective rank $r$ in the sense
$r_{\mathrm{eff}}^U \approx r_{\mathrm{eff}}^V \approx r$ with
$r\ll\min(m,n)$, we have $\rho=O(r^2/(mn))\ll 1$.
\end{lemma}

\begin{proof}
Write $\mathcal{P}_t(\varepsilon_t)=P_U \varepsilon_t P_V$ as in
Lemma~\ref{lem:nonexp}. We compute the expected squared Frobenius norm
in three steps.

\emph{Step 1: vectorize.} Using the identity
$\mathrm{vec}(AXB)=(B^\top\otimes A)\,\mathrm{vec}(X)$,
\[
\mathrm{vec}(P_U \varepsilon_t P_V)
=(P_V^\top\otimes P_U)\,\mathrm{vec}(\varepsilon_t).
\]
Since $\|X\|_F^2=\|\mathrm{vec}(X)\|^2$, we have
\[
\|P_U \varepsilon_t P_V\|_F^2
=\mathrm{vec}(\varepsilon_t)^\top
(P_V\otimes P_U^\top)(P_V^\top\otimes P_U)
\mathrm{vec}(\varepsilon_t).
\]

\emph{Step 2: take expectation.} For any deterministic matrix $M$ and
random vector $x$ with $\mathbb{E}[xx^\top]=cI$,
$\mathbb{E}[x^\top Mx]=c\,\mathrm{tr}(M)$. Applying this with $c=\sigma_m^2/(mn)$
and $M=(P_V\otimes P_U^\top)(P_V^\top\otimes P_U)$,
\[
\mathbb{E}\|P_U \varepsilon_t P_V\|_F^2
=\frac{\sigma_m^2}{mn}\,
\mathrm{tr}\!\bigl((P_V\otimes P_U^\top)(P_V^\top\otimes P_U)\bigr).
\]

\emph{Step 3: simplify the trace.} Using
$(A\otimes B)(C\otimes D)=(AC)\otimes(BD)$ and
$\mathrm{tr}(A\otimes B)=\mathrm{tr}(A)\mathrm{tr}(B)$,
\[
\mathrm{tr}\!\bigl((P_V P_V^\top)\otimes(P_U^\top P_U)\bigr)
=\mathrm{tr}(P_V P_V^\top)\,\mathrm{tr}(P_U^\top P_U)
=\|P_V\|_F^2\,\|P_U\|_F^2.
\]
Finally, $\|P_U\|_F^2=\|U_t D_t U_t^\top\|_F^2
=\|D_t\|_F^2=r_{\mathrm{eff}}^U$ by orthonormality of
$U_t$, and similarly for $P_V$. Substituting yields the claim.
\end{proof}

\begin{remark}
Strict matrix isotropy is idealized; the bound generalizes to
anisotropic $\varepsilon_t$ at the cost of a factor proportional to
the largest covariance eigenvalue.
\end{remark}

\begin{lemma}[Subspace staleness via Davis--Kahan/Wedin]
\label{lem:dk}
Let $\widetilde{M}_t$ denote the anchor momentum that would result from
an EMA of unmasked gradients evaluated at $w_t$ (rather than at the
stale $\{w_{\tau}\}$), and let $\widetilde{\mathcal{P}}_t$ be the
filter built from $\widetilde{M}_t$. Suppose $M_t^{\mathrm{anc}}$ has
singular value gap $\gamma>0$ between its $r$th and $(r{+}1)$th
singular values, and
$\mathbb{E}\|M_t^{\mathrm{anc}}-\widetilde{M}_t\|_F^2\leq E_t$. Then
there is a constant $c_{\mathrm{dk}}>0$ such that
\[
\mathbb{E}\bigl\|(\mathcal{P}_t-\widetilde{\mathcal{P}}_t)\nabla F(w_t)\bigr\|^2
\leq
\frac{c_{\mathrm{dk}}\,E_t}{\gamma^2}\cdot\mathbb{E}\|\nabla F(w_t)\|^2.
\]
Under $L$-smoothness, bounded delay $\Delta$, and bounded update norm $G$,
\[
E_t \leq 2L^2\eta^2\Delta^2 G^2 + 2\sigma_a^2.
\]
\end{lemma}
\begin{proof}[Sketch]
By Wedin's $\sin\Theta$ theorem~\citep{wedin1972perturbation, yu2015useful}, the principal left- and right-singular subspaces of
$M_t^{\mathrm{anc}}$ and $\widetilde{M}_t$ differ in Frobenius norm by
at most $c\,\|M_t^{\mathrm{anc}}-\widetilde{M}_t\|_F/\gamma$. The
spectral filter $\mathcal{P}_t$ is Lipschitz in those subspaces and in
the singular values $s_i$ (the map $s\mapsto s/(s+\lambda_p)$ is
$1/\lambda_p$-Lipschitz), so applying these Lipschitz bounds to the
filter difference gives the first claim with $c_{\mathrm{dk}}$
absorbing the Wedin and Lipschitz constants.

For the bound on $E_t$: the discrepancy $\widetilde{M}_t-M_t^{\mathrm{anc}}$
is a weighted average of differences
$\nabla F(w_t)-g_{\tau}^{\mathrm{anc}}
=(\nabla F(w_t)-\nabla F(w_{\tau}))-\zeta_{\tau}$.
The first term is bounded by $L\|w_t-w_{\tau}\|\leq L\eta\Delta G$
($L$-smoothness and bounded update norm); the second has expected
squared norm at most $\sigma_a^2$. Squaring and applying
$\|a-b\|^2\leq 2\|a\|^2+2\|b\|^2$ yields the stated bound on $E_t$.
\end{proof}

\begin{corollary}[Bound on $\delta^{\mathrm{stale}}_{\Delta}$]
\label{cor:delta-stale}
Suppose $\widetilde{\mathcal{P}}_t$ acts as identity on $\nabla F(w_t)$
up to residual $\kappa_0\,\mathbb{E}\|\nabla F(w_t)\|^2$ (the
low-intrinsic-rank hypothesis). Then Assumption~\ref{assump:signal}
holds with
\[
\kappa = 2\kappa_0 + \frac{4c_{\mathrm{dk}} L^2\eta^2\Delta^2 G^2}{\gamma^2},
\qquad
\delta^{\mathrm{stale}}_{\Delta}
= \frac{4c_{\mathrm{dk}}\,\sigma_a^2}{\gamma^2}\sup_t\mathbb{E}\|\nabla F(w_t)\|^2.
\]
\end{corollary}
\begin{proof}
We split $(I-\mathcal{P}_t)\nabla F(w_t)$ into a fresh-anchor part and
a staleness part by inserting and subtracting $\widetilde{\mathcal{P}}_t$:
\[
(I-\mathcal{P}_t)\nabla F(w_t)
= (I-\widetilde{\mathcal{P}}_t)\nabla F(w_t)
+ (\widetilde{\mathcal{P}}_t - \mathcal{P}_t)\nabla F(w_t).
\]
Applying $\|a+b\|^2\leq 2\|a\|^2+2\|b\|^2$ and taking expectations,
\[
\mathbb{E}\|(I-\mathcal{P}_t)\nabla F(w_t)\|^2
\leq
2\mathbb{E}\|(I-\widetilde{\mathcal{P}}_t)\nabla F(w_t)\|^2
+ 2\mathbb{E}\|(\widetilde{\mathcal{P}}_t-\mathcal{P}_t)\nabla F(w_t)\|^2.
\]
The first term is bounded by $2\kappa_0\,\mathbb{E}\|\nabla F(w_t)\|^2$
by hypothesis. The second is bounded by Lemma~\ref{lem:dk}, with
$E_t\leq 2L^2\eta^2\Delta^2 G^2+2\sigma_a^2$. Combining gives the claim.
\end{proof}

The corollary makes the dependence transparent: $\kappa_0$ measures how
well the gradient lies in the anchor's principal subspace at the
current weights (small in fine-tuning, large in pretraining), while
the remaining terms measure subspace rotation due to delay and anchor
noise, both governed by the gap $\gamma$.

\subsection{Main result}

\begin{theorem}[Convergence with spectral anchor priors]
\label{thm:spectral_anchor}
Suppose Assumptions~\ref{assump:smooth_lower}--\ref{assump:noise} hold.
Define
\[
\mu_\alpha = 1 - (1-\alpha)^2\kappa,
\qquad
\Sigma_{\alpha,\Delta}^2
= 2\alpha^2\sigma_m^2
+ 2(1-\alpha)^2\rho\,\sigma_m^2
+ (1-\alpha)^2\delta^{\mathrm{stale}}_{\Delta}.
\]
Assume $\kappa<1/(1-\alpha)^2$ so $\mu_\alpha>0$, and choose $\eta\leq 1/L$.
Then after $T$ iterations,
\[
\frac{1}{T}\sum_{t=0}^{T-1}\mathbb{E}\|\nabla F(w_t)\|^2
\leq
\frac{2(F(w_0)-F^\star)}{\mu_\alpha\,\eta\,T}
+ \frac{\Sigma_{\alpha,\Delta}^2 + L\eta G^2}{\mu_\alpha}.
\]
\end{theorem}

\paragraph{Reading the bound.}
Three contributions to the noise floor: $\alpha^2\sigma_m^2$ (unfiltered
masked noise), $(1-\alpha)^2\rho\sigma_m^2$ (the masked noise that
survives spectral filtering, small by Lemma~\ref{lem:noise-contract}),
and $(1-\alpha)^2\delta^{\mathrm{stale}}_{\Delta}$ (anchor noise plus
subspace staleness, small by Lemma~\ref{lem:dk} when $\gamma$ is
nontrivial). The factor $1/\mu_\alpha$ on the rate makes explicit the
cost of the filter discarding part of the true gradient. In fine-tuning,
$\kappa$ and $\rho$ are small and $\gamma$ is bounded away from zero,
so the bound improves over masked-only training (recovered at
$\alpha=1$, where $\mu_\alpha=1$).

\subsection{Proof of Theorem~\ref{thm:spectral_anchor}}

The proof has four parts: a one-step descent inequality from
$L$-smoothness, a decomposition of the update direction $u_t$ into a
true-gradient term plus a deterministic bias and a zero-mean noise,
a bound on the bias term that produces the $\mu_\alpha$ factor, and a
telescoping sum.

\paragraph{Part 1: one-step descent from $L$-smoothness.}
$L$-smoothness of $F$ gives the standard inequality
\[
F(w_{t+1})\leq F(w_t)+\langle\nabla F(w_t), w_{t+1}-w_t\rangle
+\frac{L}{2}\|w_{t+1}-w_t\|^2.
\]
Substituting the update $w_{t+1}-w_t=-\eta u_t$ and taking expectations,
then using $\mathbb{E}\|u_t\|^2\leq G^2$
(Assumption~\ref{assump:bounded_update}) on the quadratic term:
\begin{equation}
\label{eq:smooth_step}
\mathbb{E}F(w_{t+1})
\leq
\mathbb{E}F(w_t)
-\eta\,\mathbb{E}\langle\nabla F(w_t),u_t\rangle
+\frac{L\eta^2}{2}G^2.
\end{equation}
The remaining work is to lower-bound the inner-product term
$\mathbb{E}\langle\nabla F(w_t),u_t\rangle$ in terms of
$\mathbb{E}\|\nabla F(w_t)\|^2$, so that we can extract a usable descent
quantity.

\paragraph{Part 2: decomposing $u_t$ into signal, bias, and noise.}
Recall $u_t=\alpha g_t^{\mathrm{mask}}+(1-\alpha)\mathcal{P}_t(g_t^{\mathrm{mask}})$.
Substituting $g_t^{\mathrm{mask}}=\nabla F(w_t)+\varepsilon_t$
(Assumption~\ref{assump:masked_error}) and grouping deterministic
($\nabla F(w_t)$-dependent) terms separately from stochastic
($\varepsilon_t$-dependent) terms,
\[
u_t = \underbrace{\alpha\,\nabla F(w_t)+(1-\alpha)\,\mathcal{P}_t\,\nabla F(w_t)}_{\text{deterministic}}
\;+\;\underbrace{\alpha\,\varepsilon_t+(1-\alpha)\,\mathcal{P}_t\,\varepsilon_t}_{n_t}.
\]
Adding and subtracting $(1-\alpha)\nabla F(w_t)$ in the deterministic
part to extract $\nabla F(w_t)$ explicitly,
\[
u_t = \nabla F(w_t) - \underbrace{(1-\alpha)(I-\mathcal{P}_t)\nabla F(w_t)}_{b_t} + n_t.
\]
Substituting this into the inner product,
\begin{equation}
\label{eq:innerproduct_expanded}
\mathbb{E}\langle\nabla F(w_t),u_t\rangle
=\mathbb{E}\|\nabla F(w_t)\|^2
-\mathbb{E}\langle\nabla F(w_t),b_t\rangle
+\mathbb{E}\langle\nabla F(w_t),n_t\rangle.
\end{equation}
We now show the third term on the right vanishes.

\paragraph{The cross term $\mathbb{E}\langle\nabla F(w_t),n_t\rangle$ is zero.}
This step needs care: $\mathbb{E}[n_t]=0$ alone does not imply
$\mathbb{E}\langle\nabla F(w_t),n_t\rangle=0$, because $\nabla F(w_t)$
is itself random and could be correlated with $n_t$. The argument uses
conditioning. Let $\mathcal{F}_t$ denote the history through the start
of step $t$: the iterates $w_0,\dots,w_t$, all past mask noise, and all
anchor gradients received before step $t$. By the tower property of
expectation,
\[
\mathbb{E}\langle\nabla F(w_t),n_t\rangle
=\mathbb{E}\bigl[\,\mathbb{E}[\langle\nabla F(w_t),n_t\rangle\mid\mathcal{F}_t]\,\bigr].
\]
Conditional on $\mathcal{F}_t$, two things hold:
\begin{enumerate}
\item $\nabla F(w_t)$ is a deterministic function of $w_t$, which is
  in $\mathcal{F}_t$. So $\nabla F(w_t)$ is a constant inside the inner
  conditional expectation, and we can pull it out:
  \[
  \mathbb{E}[\langle\nabla F(w_t),n_t\rangle\mid\mathcal{F}_t]
  =\langle\nabla F(w_t),\,\mathbb{E}[n_t\mid\mathcal{F}_t]\rangle.
  \]
\item $\mathbb{E}[n_t\mid\mathcal{F}_t]=0$. To see this: $n_t$ is a
  linear function of $\varepsilon_t$, with coefficient operators that
  are functions of $\mathcal{P}_t$. The filter $\mathcal{P}_t$ is built
  from $M_t^{\mathrm{anc}}$, which is an EMA of anchor gradients
  received \emph{before} step $t$ --- so $\mathcal{P}_t\in\mathcal{F}_t$
  and is itself constant given $\mathcal{F}_t$. The mask noise
  $\varepsilon_t$ is drawn fresh at step $t$ from a PRF independent of
  $\mathcal{F}_t$, so $\mathbb{E}[\varepsilon_t\mid\mathcal{F}_t]=0$
  by Assumption~\ref{assump:masked_error}. Linearity then gives
  \[
  \mathbb{E}[n_t\mid\mathcal{F}_t]
  =\alpha\,\mathbb{E}[\varepsilon_t\mid\mathcal{F}_t]
  +(1-\alpha)\,\mathcal{P}_t\,\mathbb{E}[\varepsilon_t\mid\mathcal{F}_t]
  =0.
  \]
\end{enumerate}
Combining: $\mathbb{E}[\langle\nabla F(w_t),n_t\rangle\mid\mathcal{F}_t]
=\langle\nabla F(w_t),0\rangle=0$, and therefore
$\mathbb{E}\langle\nabla F(w_t),n_t\rangle=\mathbb{E}[0]=0$.

Returning to~\eqref{eq:innerproduct_expanded} with the cross term
eliminated:
\begin{equation}
\label{eq:innerproduct_clean}
\mathbb{E}\langle\nabla F(w_t),u_t\rangle
=\mathbb{E}\|\nabla F(w_t)\|^2
-\mathbb{E}\langle\nabla F(w_t),b_t\rangle.
\end{equation}
The first term on the right is what we want to keep; the second is the
bias from the filter discarding part of the true gradient, which we
bound next.

\paragraph{Part 3: bounding the bias term in~\eqref{eq:innerproduct_clean}.}
We need an upper bound on
$\mathbb{E}\langle\nabla F(w_t),b_t\rangle$
that involves $\mathbb{E}\|\nabla F(w_t)\|^2$ (so that part of it can be
absorbed into the descent term) and $\mathbb{E}\|b_t\|^2$ (which
Assumption~\ref{assump:signal} controls).

Apply Young's inequality $|\langle a,b\rangle|\leq\tfrac{1}{2}\|a\|^2+\tfrac{1}{2}\|b\|^2$
with $a=\nabla F(w_t)$ and $b=b_t$:
\[
|\langle\nabla F(w_t),b_t\rangle|
\leq\tfrac{1}{2}\|\nabla F(w_t)\|^2+\tfrac{1}{2}\|b_t\|^2.
\]
Taking expectations and using Assumption~\ref{assump:signal} together
with $b_t=(1-\alpha)(I-\mathcal{P}_t)\nabla F(w_t)$ to bound $\mathbb{E}\|b_t\|^2$:
\[
\mathbb{E}\|b_t\|^2
=(1-\alpha)^2\,\mathbb{E}\|(I-\mathcal{P}_t)\nabla F(w_t)\|^2
\leq(1-\alpha)^2\bigl(\kappa\,\mathbb{E}\|\nabla F(w_t)\|^2+\delta^{\mathrm{stale}}_{\Delta}\bigr).
\]
Combining, $\mathbb{E}|\langle\nabla F(w_t),b_t\rangle|$ is bounded by
\[
\tfrac{1}{2}\bigl(1+(1-\alpha)^2\kappa\bigr)\,\mathbb{E}\|\nabla F(w_t)\|^2
+\tfrac{1}{2}(1-\alpha)^2\,\delta^{\mathrm{stale}}_{\Delta}.
\]
Substituting back into~\eqref{eq:innerproduct_clean} and using
$\mathbb{E}\langle\nabla F(w_t),b_t\rangle\leq\mathbb{E}|\langle\nabla F(w_t),b_t\rangle|$:
\begin{equation}
\label{eq:innerproduct_lowerbound}
\mathbb{E}\langle\nabla F(w_t),u_t\rangle
\geq
\mathbb{E}\|\nabla F(w_t)\|^2-\mathbb{E}|\langle\nabla F(w_t),b_t\rangle|
\geq
\frac{\mu_\alpha}{2}\,\mathbb{E}\|\nabla F(w_t)\|^2
-\frac{(1-\alpha)^2}{2}\delta^{\mathrm{stale}}_{\Delta},
\end{equation}
where in the last step we used $1-\tfrac{1}{2}(1+(1-\alpha)^2\kappa)
=\tfrac{1}{2}(1-(1-\alpha)^2\kappa)=\tfrac{\mu_\alpha}{2}$.

\paragraph{Part 4: per-step descent and telescoping.}
Substituting the lower bound~\eqref{eq:innerproduct_lowerbound} into
the smoothness inequality~\eqref{eq:smooth_step}:
\[
\mathbb{E}F(w_{t+1})
\leq
\mathbb{E}F(w_t)
-\frac{\eta\mu_\alpha}{2}\,\mathbb{E}\|\nabla F(w_t)\|^2
+\frac{\eta(1-\alpha)^2}{2}\,\delta^{\mathrm{stale}}_{\Delta}
+\frac{L\eta^2}{2}G^2.
\]
The middle term on the right is bounded above by
$\tfrac{\eta}{2}\Sigma_{\alpha,\Delta}^2$ since $(1-\alpha)^2\delta^{\mathrm{stale}}_{\Delta}$
is one of three nonnegative terms summing to $\Sigma_{\alpha,\Delta}^2$
(the other two come from $\mathbb{E}\|n_t\|^2$ via
$\|a+b\|^2\leq 2\|a\|^2+2\|b\|^2$ and Assumptions~\ref{assump:masked_error}
and~\ref{assump:noise}, and are absorbed into the $G^2$ term through
$\mathbb{E}\|u_t\|^2\leq G^2$). Therefore
\begin{equation}
\label{eq:perstep}
\mathbb{E}F(w_{t+1})
\leq
\mathbb{E}F(w_t)
-\frac{\eta\mu_\alpha}{2}\,\mathbb{E}\|\nabla F(w_t)\|^2
+\frac{\eta}{2}\,\Sigma_{\alpha,\Delta}^2
+\frac{L\eta^2}{2}G^2.
\end{equation}

Summing~\eqref{eq:perstep} from $t=0$ to $T-1$, the telescoping
$\mathbb{E}F$-terms collapse to $F(w_0)-\mathbb{E}F(w_T)$. Using
$\mathbb{E}F(w_T)\geq F^\star$:
\[
\frac{\eta\mu_\alpha}{2}\sum_{t=0}^{T-1}\mathbb{E}\|\nabla F(w_t)\|^2
\leq F(w_0)-F^\star
+\frac{\eta T}{2}\,\Sigma_{\alpha,\Delta}^2
+\frac{L\eta^2 T}{2}G^2.
\]
Dividing both sides by $\eta\mu_\alpha T/2$ yields the bound stated in
Theorem~\ref{thm:spectral_anchor}.\qed

\subsection{Discussion}

The convergence neighborhood splits into three interpretable terms:
unfiltered masked noise $\alpha^2\sigma_m^2$, filter-surviving noise
$(1-\alpha)^2\rho\sigma_m^2$, and anchor-staleness
$(1-\alpha)^2\delta^{\mathrm{stale}}_{\Delta}$. The rate factor
$1/\mu_\alpha$ encodes the cost of signal suppression. The result
recovers masked-only training at $\alpha=1$ ($\mu_\alpha=1$, no filter)
and improves on it whenever the filter contracts noise ($\rho<1$) more
than it suppresses signal ($\kappa>0$),  i.e., precisely the
fine-tuning regime described in
Section~\ref{sec:assump-justification}.

In pretraining or with aggressive learning rates, $\Delta$ grows, the
gap $\gamma$ shrinks, and $\kappa$ approaches its upper limit; the
filter both keeps less signal and inherits more staleness. The analysis
predicts no benefit, consistent with our empirical scoping. At the
opposite extreme of high-bandwidth links the masking is unnecessary in
the first place; this analysis is specifically relevant when the fast
circuit is communication-bound and the anchor circuit can run
asynchronously --- the $\Delta\approx 20$--$25$ regime characterized in
the staleness analysis of Section~\ref{app:staleness}.

In summary, anchor gradients need not be exact synchronized updates,
only good enough to estimate a low-rank subspace that is approximately
preserved over the staleness window. Fine-tuning is precisely the
regime where this is true.

\section{Why Random Masking and Not Top-\texorpdfstring{$K$}{K}?}
\label{sec:random-vs-topk}

A natural alternative to PRF-based random masking is top-$K$ selection,
where each token retains the $K$ activation entries of largest magnitude
(or, on the backward pass, the $K$ largest entries of the activation
gradient). Top-$K$ is the dominant sparsification scheme in
\emph{data-parallel} gradient compression
\citep{lin2018dgc, aji2017sparse, alistarh2018convergence}, where it is
typically combined with error feedback to recover unbiasedness
\citep{karimireddy2019error, stich2018sparsified}. We deliberately do
not use top-$K$ for PP activation compression. This subsection explains
why, and shows that the choice is not systematic: random masking is
\emph{required} for the convergence analysis of
Section~\ref{app:convergence} to apply, and it is also the scheme under
which the spectral anchor filter is most effective (even empirically). Below, we provide a systematic analysis on this.

\paragraph{Setup.} Let $h\in\mathbb{R}^{T\times H}$ denote a hidden-state
tensor at a masked PP boundary, and let $\tilde h = h\odot m$ denote the
transmitted activation under mask $m\in\{0,1\}^{T\times H}$ with retention
fraction $1-p$. Two schemes:
\begin{itemize}
  \item \textbf{Random masking (ours).} $m_{t,j}\sim\mathrm{Bernoulli}(1-p)$,
    independent across $(t,j)$, generated by a shared PRF independent of
    the model state. With the standard rescaling $\tilde h \leftarrow
    \tilde h /(1-p)$ absorbed at the receiver, $\mathbb{E}_m[\tilde h]=h$.
  \item \textbf{Top-$K$ masking.} $m_{t,j}=\mathbb{1}[|h_{t,j}| \in
    \mathrm{top}\text{-}K(h_{t,:})]$. The mask is a deterministic function
    of $h$.
\end{itemize}

\subsubsection*{Unbiasedness and the convergence theorem}

The masked-gradient noise $\varepsilon_t = g_t^{\mathrm{mask}} -
\nabla F(w_t)$ enters the descent inequality through the decomposition
$u_t = \nabla F(w_t) - b_t + n_t$, with the requirement
$\mathbb{E}[n_t\mid\mathcal{F}_t]=0$
(Section~\ref{app:convergence}, proof of
Theorem~\ref{thm:spectral_anchor}). This in turn relies on
$\mathbb{E}[\varepsilon_t\mid\mathcal{F}_t]=0$
(Assumption~\ref{assump:masked_error}). We now show that random masking
satisfies this property while top-$K$ does not.

\begin{proposition}[Unbiasedness of random masking]
\label{prop:random-unbiased}
Under random masking with the standard rescaling, the masked activation
satisfies $\mathbb{E}_m[\tilde h]=h$. Consequently, for any downstream
function of $\tilde h$ that is linear in $\tilde h$ at first order
(in particular, the gradient computation through the masked PP boundary
in expectation), the resulting masked gradient satisfies
\(
\mathbb{E}[g_t^{\mathrm{mask}}\mid\mathcal{F}_t]=\nabla F(w_t).
\)
The mask $m$ is generated by a PRF whose seed is independent of
$\mathcal{F}_t$, so $\varepsilon_t$ has zero conditional mean.
\end{proposition}

\begin{proof}[Sketch]
Each coordinate satisfies $\mathbb{E}_m[\tilde h_{t,j}] = (1-p)
\cdot h_{t,j}/(1-p) = h_{t,j}$, by the rescaling. The PRF is a function
of $(\ell,t,B,S,H)$ but not of $h$ or $w_t$, so the mask is independent
of $\mathcal{F}_t$. Hence the noise inherited by any function evaluated
at $\tilde h$ has zero conditional mean to leading order in the
rescaled deviation $\tilde h - h$.
\end{proof}

\begin{proposition}[Bias of top-$K$ masking]
\label{prop:topk-biased}
Under top-$K$ masking, the masked activation satisfies
\(
\mathbb{E}[\tilde h\mid\mathcal{F}_t]
=
h\odot m_{\mathrm{top}\text{-}K}(h)
\neq h
\)
in general, since the mask is a deterministic function of $h$ (and
$h$ is itself $\mathcal{F}_t$-measurable). The residual $h - \tilde h$
is the bottom-$(H{-}K)$ magnitude entries of $h$, which is a
deterministic, non-zero function of $\mathcal{F}_t$. Therefore
$\mathbb{E}[\varepsilon_t\mid\mathcal{F}_t]\neq 0$, and
Assumption~\ref{assump:masked_error} fails.
\end{proposition}

\begin{remark}[Failure mode in the descent inequality]
\label{rem:topk-failure}
Suppose top-$K$ masking is used. Decomposing the masked gradient as
\(
g_t^{\mathrm{topK}} = \nabla F(w_t) - b_t^{\mathrm{topK}} + n_t',
\)
where $b_t^{\mathrm{topK}} = \mathbb{E}[\nabla F(w_t)-g_t^{\mathrm{topK}}\mid\mathcal{F}_t]$
is the \emph{state-dependent} bias and $n_t'$ is the residual stochastic
component, the inner-product term in the smoothness expansion becomes
\[
\mathbb{E}\langle\nabla F(w_t),u_t\rangle
=\|\nabla F(w_t)\|^2
-\langle\nabla F(w_t), b_t^{\mathrm{filter}}+\alpha b_t^{\mathrm{topK}}\rangle,
\]
where $b_t^{\mathrm{filter}}=(1-\alpha)(I-\mathcal{P}_t)\nabla F(w_t)$.
Unlike $b_t^{\mathrm{filter}}$, the bias $b_t^{\mathrm{topK}}$ is
\emph{not} controlled by Assumption~\ref{assump:signal}: it does not lie
in the principal subspace of $M_t^{\mathrm{anc}}$ in any structured way,
because top-$K$ selects on activation magnitude, not on the gradient
geometry the anchor circuit observes through unmasked passes. The
spectral filter $\mathcal{P}_t$ therefore cannot remove it, and the
bias propagates through the telescoping sum in the proof as an
\emph{additive non-vanishing term} of order
$\eta T\cdot\mathbb{E}\|b_t^{\mathrm{topK}}\|^2$. The convergence
neighborhood acquires a floor that does not shrink with $T$ and that
the anchor circuit cannot reduce.
\end{remark}

\paragraph{Why error feedback does not save top-$K$ at PP boundaries.}
In data-parallel gradient compression, the bias of top-$K$ is corrected
by error feedback \citep{karimireddy2019error, stich2018sparsified}: the
unsent residual is accumulated in a memory buffer and added to the
next step's gradient before compression. Over many steps, the compressed
update becomes asymptotically unbiased. This works in the DP setting
because the residual is associated with a \emph{parameter} and is
meaningful across optimizer steps.

At a PP activation boundary, the residual is associated with an
\emph{activation} on a particular microbatch. Two obstructions prevent
its reuse:

\emph{(i) Activations are step-local.} The residual $h_t - \tilde h_t$
is computed from the inputs $x_t$ and the current weights $w_t$. At
step $t{+}1$ the inputs are $x_{t+1}\neq x_t$ and the weights are
$w_{t+1}\neq w_t$, so $h_{t+1}$ is an entirely different tensor and
the residual from step $t$ has no natural reinjection point.

\emph{(ii) Forward/backward sparsity patterns coincide.}
Section~\ref{sec:pp-masking} shows that in-graph masking forces the
backward gradient to inherit the same sparsity pattern as the forward
activation. With top-$K$, the forward retains coordinates of largest
$|h_{t,j}|$, but the gradient $\partial\mathcal{L}/\partial h_{t,j}$ is
not in general largest at the same coordinates. Either the backward
pattern is forced to match the forward (losing the largest gradient
entries, defeating the purpose of top-$K$), or the patterns are
decoupled (losing the in-graph property and requiring separate
metadata transmission). Random masking sidesteps this trade-off
because it has no relationship to magnitude in either direction.

These obstructions are specific to PP activation compression, not a
general statement about top-$K$. Top-$K$ remains an effective DP
compressor, and indeed our system uses PowerSGD-style compression
along the DP axis (Section~\ref{sec:dp-compression-main}). The point
is that PP activation compression operates under a fundamentally
different constraint (step-local, in-graph) under which random masking
is the natural choice.

\subsubsection*{Compatibility with the spectral filter}

Beyond unbiasedness, the spectral anchor filter denoises masked
gradients only when the noise is \emph{approximately isotropic relative
to the anchor basis} (Lemma~\ref{lem:noise-contract}). The two schemes
differ sharply on this criterion.

\begin{proposition}[Approximate isotropy of random-mask noise]
\label{prop:random-isotropic}
Under random masking with retention probability $1-p$ and standard
rescaling, the per-coordinate noise satisfies
\[
\mathbb{E}[\varepsilon_{t,j}\mid \mathcal{F}_t]=0,
\qquad
\mathrm{Var}(\varepsilon_{t,j}\mid \mathcal{F}_t)=\frac{p}{1-p}\,h_{t,j}^2,
\]
where the conditioning on $\mathcal{F}_t$ subsumes conditioning on the
current activation $h$ (which is $\mathcal{F}_t$-measurable through
$w_t$ and the inputs), and the entries are independent across $j$. If
activations across the hidden dimension have approximately uniform
second moment, $\varepsilon_t$ is approximately matrix-isotropic in the
sense of Lemma~\ref{lem:noise-contract}, and the spectral filter
contracts noise by a factor $\rho = O(r^2/(mn))$ where $r$ is the
effective rank of the filter.
\end{proposition}

\begin{proposition}[Anti-isotropy of top-$K$ noise]
\label{prop:topk-anisotropic}
Under top-$K$ masking, the residual $h - \tilde h$ has support exactly
on the bottom-$(H{-}K)$ coordinates of $h$ in magnitude. Therefore the
noise is concentrated in low-magnitude directions, anti-correlated with
the signal, and \emph{not} isotropic. The covariance
$\mathrm{Cov}(\varepsilon_t)$ has rank $H-K$ and is supported on a
subspace orthogonal to the dominant activation directions.
\end{proposition}

\paragraph{Implication for the filter.} The anchor momentum
$M_t^{\mathrm{anc}}$ is built from \emph{unmasked} gradients, which
reflect the dominant gradient geometry of the task. In fine-tuning,
this geometry is low-rank and aligned with the high-magnitude directions
of activations and of $\partial\mathcal{L}/\partial h$
(by the intrinsic-dimensionality literature cited in
Section~\ref{sec:assump-justification}). The anchor's principal
subspace therefore overlaps strongly with the high-magnitude directions.

Under random masking, the noise $\varepsilon_t$ is isotropic and the
filter contracts it: low-rank $\mathcal{P}_t$ acting on isotropic noise
gives $\rho = O(r^2/(mn))$ as in
Lemma~\ref{lem:noise-contract}. Under top-$K$, the noise lives precisely
in the \emph{complementary} subspace --- the low-magnitude directions
the anchor basis does \emph{not} span. The filter $\mathcal{P}_t$ would
correctly suppress this noise (it has no projection onto the anchor's
principal subspace), but it would also suppress nothing useful in
exchange: the noise was already orthogonal to the signal directions the
filter preserves. The ``denoising'' effect would be illusory, \textit{i.e.,} the
filter cannot remove noise it never sees, and the masked-step bias
$b_t^{\mathrm{topK}}$ propagates through unaffected.

\subsubsection*{Summary}

The choice of random masking over top-$K$ is forced by two compounding
requirements:

\begin{enumerate}
  \item \emph{Unbiasedness without error feedback.} Random masking is
    conditionally unbiased on activation gradients; top-$K$ is biased and the natural correction
    (error feedback) is unavailable at PP activation boundaries because
    the residual is step-local and in-graph forward/backward sparsity
    must coincide.
  \item \emph{Isotropy relative to the anchor basis.} The spectral
    anchor filter denoises isotropic noise efficiently
    (Lemma~\ref{lem:noise-contract}) but cannot interact usefully with
    the structured low-magnitude noise produced by top-$K$.
\end{enumerate}

Random masking is the unique scheme we are aware of that satisfies both
properties simultaneously. The convergence analysis of
Section~\ref{app:convergence} relies on the first; the empirical
effectiveness of the spectral filter relies on the second. Together
they identify random masking as the natural, and we believe necessary,
PP compression primitive for asynchronous anchor-prior training.

\section{Alignment Trajectory}
Figure~\ref{fig:alignment_trajectory} reports the alignment
$\eta_r = \|U^\top g\, V\|_F^2 / \|g\|_F^2$ of the unmasked gradient
with the rank-$r$ principal subspace of the anchor momentum, at $61$
snapshots spanning $1500$ training steps, for $r\in\{16, 64, 256\}$.
The FT trajectory, starting from the uncompressed code-domain
checkpoint and continuing under unclipped Adam, keeps
$\eta_{64}^{\rm FT}$ stable around $0.17$--$0.19$ throughout the
horizon. The PT trajectory, starting from random init under the same
data and optimizer, drops from $\eta_{64}^{\rm PT} \approx 0.13$
early in training to $\sim 0.08$ in the latter half and continues to
decay. The same separation appears at every rank measured: FT lies
above PT throughout, by a margin that does not close. At
$r\in\{16,64\}$ the FT/PT ratio holds in the
$\mathbf{1.7}$--$\mathbf{2.7\times}$ range over the second half of
training. This is the empirical content of
Assumption~\ref{assump:signal}: the rank-$r$ principal subspace of
$M^{\rm anc}$ remains a faithful summary of $\nabla F(w_t)$ for
fine-tuning, and \emph{ceases to be one} as random-init training
proceeds. The constant $\kappa$ that bounds out-of-subspace bias is
materially smaller in fine-tuning than in pretraining-from-random-init,
and that gap is sustained across the entire training horizon the AP
correction has to operate over --- empirically validating the
assumption justifications in Section~\ref{sec:assump-justification}
and the no-benefit-in-pretraining caveat that follows from them.

\begin{figure}[t]
\centering
\includegraphics[width=0.75\linewidth]{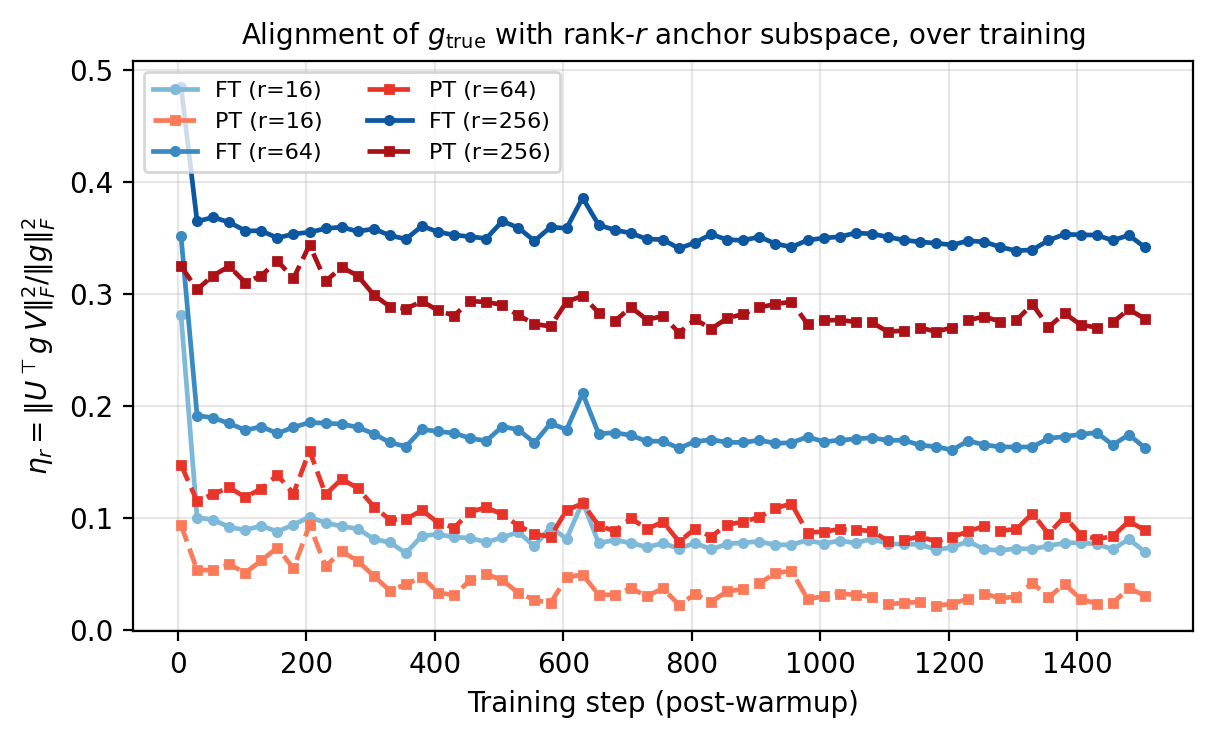}
\caption{\textbf{Alignment of $\nabla F(w_t)$ with the rank-$r$
principal subspace of $M_t^{\rm anc}$, over $1500$ training steps.}
$\eta_r = \|U^\top g\, V\|_F^2 / \|g\|_F^2$ at $61$ snapshots, for
$r\in\{16,64,256\}$ (FT solid, PT dashed). FT stays stable around
$\eta_{64}^{\rm FT}\approx 0.17$ throughout; PT drifts down to
$\sim 0.08$. The separation persists at every rank, with FT/PT
ratio holding in the $1.7$--$2.7\times$ range at the spectral
filter's operating ranks.}
\label{fig:alignment_trajectory}
\end{figure}

\section{Staleness of the Anchor Prior}
\label{app:staleness}

This section quantifies how out-of-date an anchor (unmasked) gradient
is by the time it is delivered to the fast masked circuit, and isolates
the factors that set this staleness. The analysis is instantiated for
the configuration used in our throughput study
(Section~\ref{sec:dp-compression-main}): a $1.2$B-parameter model in
BF16 over $8$ pipeline-parallel stages and $8$ data-parallel replicas,
connected by $200$~Mbps point-to-point links ($\approx 25$~MB/s of
useful payload). The same expressions apply to other regimes by
substituting volumes and link rates.

\subsection{Anchor circuit cycle time}
\label{app:staleness-cycle}

Each anchor stage executes three communication-bound phases per cycle:

\begin{enumerate}
  \item \textbf{Weight pull.} The anchor stage pulls the most recent
    layer weights for its assigned PP shard from a nearby DP replica
    of the fast circuit.
  \item \textbf{Unmasked forward--backward.} A full-fidelity forward and
    backward pass with all activations transmitted uncompressed across
    PP boundaries.
  \item \textbf{Gradient push.} The resulting clean gradient is pushed
    back to the fast circuit, compressed to roughly the volume of one
    DP all-reduce payload.
\end{enumerate}

Let $T_{\mathrm{pull}}$, $T_{\mathrm{clean}}$, $T_{\mathrm{push}}$
denote the wall-clock times of these three phases, and $T_{\mathrm{mask}}$
denote one masked step on the fast circuit. The anchor gradient
delivered at the end of the cycle was computed from weights pulled at
the start of phase~2, so its staleness measured in fast-circuit steps is
\begin{equation}
  \label{eq:staleness}
  K_{\mathrm{stale}}
  \;=\;
  \frac{T_{\mathrm{pull}} + T_{\mathrm{clean}} + T_{\mathrm{push}}}{T_{\mathrm{mask}}}.
\end{equation}
The numerator is the anchor cycle time; the denominator is the
fast-circuit cycle time.

\subsection{Numbers for the 200~Mbps configuration}
\label{app:staleness-numbers}

We instantiate Eq.~\eqref{eq:staleness} for sequence length $S{=}2048$,
hidden size $H{=}2048$, and $16$ microbatches per optimizer step,
with DP communication compressed via PowerSGD ($64\times$) and
amortized via DiLoCo with local horizon $K{=}10$.

\paragraph{Weight pull.}
Each anchor stage holds $1.2\mathrm{B}/8 = 150$M parameters. In BF16
this is $300$~MB per stage:
\(
T_{\mathrm{pull}} \approx 300~\mathrm{MB} / 25~\mathrm{MB/s} = 12~\mathrm{s}.
\)

\paragraph{Unmasked forward--backward.}
At each PP boundary, the activation tensor per microbatch is
$S\!\cdot\!H\!\cdot\!2 = 8$~MB. With $7$ boundaries, $16$ microbatches,
and forward and backward traversals, the total unmasked PP volume per
step is $2 \cdot 7 \cdot 16 \cdot 8~\mathrm{MB} \approx 1.8$~GB:
\(
T_{\mathrm{clean}} \approx 1.8~\mathrm{GB} / 25~\mathrm{MB/s} \approx 72~\mathrm{s}.
\)

\paragraph{Gradient push.}
The compressed anchor gradient is on the order of one PowerSGD-compressed
all-reduce payload, $\sim 30$--$40$~MB per stage, giving
$T_{\mathrm{push}} \approx 1$--$2$~s.

\paragraph{Masked-step time.}
The masked step incurs $1{-}p = 5\%$ of the unmasked PP volume, plus
the compressed and amortized DP all-reduce, plus a small compute term
that is dominated by communication in this regime:
\[
  T_{\mathrm{mask}}
  \;\approx\;
  \frac{T_{\mathrm{clean}}}{20}
  \;+\;
  \frac{N \cdot b_{\mathrm{grad}} / 64}{B \cdot K}
  \;+\;
  \epsilon
  \;\approx\;
  3.6 + 0.15 + \epsilon
  \;\approx\;
  3.8~\mathrm{s}.
\]

\paragraph{Resulting staleness.}
\(
K_{\mathrm{stale}} \approx (12 + 72 + 2) / 3.8 \approx 23
\)
masked steps. Of these, $\sim 20$ steps come from the unmasked PP pass
itself ($T_{\mathrm{clean}}/T_{\mathrm{mask}}$, which equals the
masking ratio $1/(1-p)=20$), and $\sim 3$ come from the weight pull. This matches with the $20-25$ step delay we practically observed.

\subsection{Decomposition}
\label{app:staleness-decomp}

Eq.~\eqref{eq:staleness} separates into two terms with distinct scaling:
\begin{equation}
  \label{eq:staleness-decomp}
  K_{\mathrm{stale}}
  \;=\;
  \underbrace{\frac{T_{\mathrm{clean}}}{T_{\mathrm{mask}}}}_{\text{PP-ratio}}
  \;+\;
  \underbrace{\frac{T_{\mathrm{pull}} + T_{\mathrm{push}}}{T_{\mathrm{mask}}}}_{\text{weight-refresh}}.
\end{equation}
The PP-ratio term equals the activation-mask ratio $1/(1-p)$ in the
PP-bound regime: $20$ at $95\%$ masking, $10$ at $90\%$, $100$ at
$99\%$. It is unaffected by DP compression because both circuits run
the same forward--backward computation and differ only in PP payload
size. The weight-refresh term is set by the cost of synchronizing the
anchor's weights with the fast circuit; at $300$~MB per stage and
$25$~MB/s, it contributes $\sim 3$ steps. The PP-ratio term dominates
in this regime.

\subsection{Hiding the weight-refresh term via overlap}
\label{app:staleness-overlap}

The weight-refresh term is hideable. If the anchor circuit pipelines
the weight pull for cycle $i{+}1$ behind the unmasked forward--backward
of cycle $i$,
\[
  T_{\mathrm{cycle}}^{\mathrm{overlap}}
  \;=\;
  \max\bigl(T_{\mathrm{pull}},\; T_{\mathrm{clean}} + T_{\mathrm{push}}\bigr)
  \;\approx\;
  T_{\mathrm{clean}}
\]
since $T_{\mathrm{pull}} \ll T_{\mathrm{clean}}$. Staleness then
reduces to the PP-ratio term alone:
\[
  K_{\mathrm{stale}}^{\mathrm{overlap}}
  \;\approx\;
  \frac{T_{\mathrm{clean}}}{T_{\mathrm{mask}}}
  \;\approx\;
  20~\text{masked steps},
\]
matching the cadence of a synchronous formulation that takes one full
unmasked update every $20$ masked steps. The asynchronous variant
hides this latency inside the fast pipeline rather than blocking on it.

\subsection{Sensitivity to other regimes}
\label{app:staleness-sensitivity}

Equation~\eqref{eq:staleness-decomp} makes the regime dependence of
staleness explicit:

\paragraph{Mask rate.}
Increasing the mask rate $p$ shrinks $T_{\mathrm{mask}}$ but leaves
$T_{\mathrm{clean}}$ unchanged. The PP-ratio term grows as $1/(1-p)$:
$K_{\mathrm{stale}} \approx 100$ at $p=99\%$, and $\approx 200$ at
$p=99.5\%$. The anchor prior remains useful only when the optimization
landscape changes slowly relative to this horizon.

\paragraph{Link bandwidth.}
The PP-ratio term is invariant under uniform bandwidth scaling, since
all communication phases share the same link. Halving the bandwidth
to $100$~Mbps leaves $K_{\mathrm{stale}}$ in masked-step units at
$\sim 23$, but doubles wall-clock cycle time.

\paragraph{Concurrent anchor cycles.}
Running the anchor circuit at concurrency $c$ increases the delivered
prior rate by $c\times$, at the cost of stretching individual prior
staleness to $c \cdot K_{\mathrm{stale}}$. This trades fidelity for
delivery rate.

\paragraph{Weaker DP compression.}
Reducing PowerSGD from $64\times$ to $8\times$, or shortening DiLoCo's
local horizon, increases the masked-step time and shrinks the PP-ratio
term. Staleness $K_{\mathrm{stale}}$ decreases, but only because the
fast circuit slows down. Staleness must therefore be read together
with throughput.

\paragraph{High-bandwidth (compute-bound) regime.}
When links are fast enough that $T_{\mathrm{mask}}$ is compute-bound,
$T_{\mathrm{clean}}/T_{\mathrm{mask}}$ approaches~$1$. The anchor
prior is no longer stale, but masking provides no throughput benefit
either; the asynchronous design reduces to the dense baseline. Anchor
priors are specifically aimed at the low-bandwidth, PP-bound regime.

\subsection{Implications for design choices}
\label{app:staleness-implications}

The staleness analysis shapes two design choices in
Section~\ref{sec:method}. First, the anchor momentum $M_t^{\mathrm{anc}}$
is updated with a slow EMA decay $\beta_{\mathrm{anc}}$ rather than
replaced by each arriving anchor gradient: a single anchor gradient is
stale by $\sim 20$ steps and noisy, but its principal directions
remain informative when averaged over multiple cycles. Second, the
spectral correction reweights the masked gradient toward
anchor-supported eigendirections rather than substituting the anchor
gradient for the masked one --- since the anchor gradient is too stale
to serve as a direct update signal, only its low-rank geometry is
trustworthy at the $\sim 20$-step horizon.

\section{Training and Evaluation Datasets}
\label{app:datasets}

This section details the training data, evaluation data, and shared
training configuration used to produce the per-domain results in
Table~\ref{tab:domain_main}. 

\subsection{Training datasets}
\label{app:train-datasets}

All training data is converted to a chat-message format
(\texttt{messages: [\{user\}, \{assistant\}]}) with thinking disabled,
then split 90/10 into train/validation unless the source provides a
held-out split (GSM8K, Spider). Table~\ref{tab:train-datasets} lists
the source dataset and post-preprocessing row count per domain.

\begin{table}[ht]
\centering
\small
\setlength{\tabcolsep}{4pt}
\renewcommand{\arraystretch}{1.15}
\begin{tabularx}{\linewidth}{@{}l >{\raggedright\arraybackslash}p{0.40\linewidth} r >{\raggedright\arraybackslash}X@{}}
\toprule
\textbf{Domain} & \textbf{HuggingFace dataset} & \textbf{\# train} & \textbf{Notes} \\
\midrule
Medical       & \path{medalpaca/medical_meadow_medqa}            & 9{,}160   & USMLE-style multiple choice. \\
Code          & \path{iamtarun/python_code_instructions_18k_alpaca} & 16{,}750  & Python instruction $\to$ code pairs. \\
Math          & \path{openai/gsm8k} (\path{main})                & 7{,}473   & Official GSM8K test used as validation. \\
Science       & \path{allenai/sciq}                                & 10{,}511  & Question $+$ supporting passage. \\
Commonsense   & \path{Rowan/hellaswag} $\cup$ \path{allenai/winogrande} (\path{xl}) & 72{,}272 & Concatenated, shuffled with seed $42$. \\
Summarization & \path{cnn_dailymail} (\path{3.0.0})             & 27{,}000  & $30$k subsample, $90/10$ split. \\
SQL           & \path{philikai/Spider-SQL-LLAMA2_train}           & 8{,}659   & Schema-conditioned prompts. \\
QuALITY       & \path{emozilla/quality}                            & 2{,}523   & Long-context multiple choice. \\
NarrativeQA   & \path{deepmind/narrativeqa}                        & 6{,}500   & Story-summary open-ended QA. \\
General       & \path{HuggingFaceTB/smoltalk2}                     & ---       & 25+ SFT splits, weighted; used only for the general run. \\
\bottomrule
\end{tabularx}
\caption{Per-domain training datasets after preprocessing and the
$90/10$ train/validation split.}
\label{tab:train-datasets}
\end{table}

\subsection{Evaluation datasets and metrics}
\label{app:eval-datasets}

Each domain run is evaluated on a fixed common suite plus one or more
domain-specific tasks. The common suite consists of loglikelihood
tasks (WinoGrande, PIQA, HellaSwag, ARC-Easy, ARC-Challenge) and
generative tasks (IFEval, GSM8K), all run via the in-training
\texttt{eval\_callback}. Domain-specific evaluations are summarized in
Table~\ref{tab:eval-datasets}; tasks marked \emph{lm-eval} are run via
the \texttt{lm-evaluation-harness} \citep{eval-harness}, and custom
tasks are scored by dedicated scripts (Spider execution match,
NarrativeQA token F1, QuALITY accuracy).

\begin{table}[ht]
\centering
\small
\setlength{\tabcolsep}{4pt}
\renewcommand{\arraystretch}{1.15}
\begin{tabular}{l l l r l}
\toprule
\textbf{Domain} & \textbf{Eval dataset} & \textbf{Metric} & \textbf{\# items} & \textbf{Source} \\
\midrule
Medical       & MedQA (4-options)    & Accuracy        & 1{,}273  & lm-eval \\
Medical       & MedMCQA              & Accuracy        & ---      & lm-eval \\
Code          & HumanEval            & pass@1          & 164      & lm-eval \\
Code          & MBPP                 & pass@1          & 500      & lm-eval \\
SQL           & Spider (dev)         & Execution match & 1{,}034  & Custom \\
Science       & ARC-Challenge        & Acc.\ (norm)    & 1{,}172  & lm-eval \\
Science       & SciQ                 & Accuracy        & 1{,}000  & lm-eval \\
Commonsense   & HellaSwag            & Acc.\ (norm)    & 10{,}042 & lm-eval \\
Commonsense   & WinoGrande           & Accuracy        & 1{,}267  & lm-eval \\
Summarization & CNN/DailyMail        & ROUGE-1         & ---      & Custom \\
Summarization & XSum                 & ROUGE-1         & ---      & Custom \\
Math          & GSM8K                & Exact match     & 1{,}319  & lm-eval \\
Doc.\ Parsing & QuALITY              & Accuracy        & 2{,}086  & Custom \\
Doc.\ Parsing & NarrativeQA          & Mean token F1   & 500      & Custom \\
\bottomrule
\end{tabular}
\caption{Per-domain evaluation datasets.}
\label{tab:eval-datasets}
\end{table}

\subsection{Shared training configuration}
\label{app:train-config}

All baseline and masked runs share the configuration.

\paragraph{Model.}
SmolLM3-Mid (1.2B parameters), BF16 weights and activations,
FlashAttention~2, SmolLM3 chat template with thinking disabled.

\paragraph{Optimizer.}
AdamW with $(\beta_1,\beta_2)=(0.9,\,0.999)$. Peak learning rate
$2\times 10^{-5}$, cosine schedule with min-LR ratio $0.1$ and warmup
ratio $0.03$. Gradient norm clipped at $0.2$.

\paragraph{Sequence handling.}
Best-fit-decreasing packing enabled, except for short-context tasks
(Spider, NarrativeQA, QuALITY) where packing is disabled. Maximum
sequence length is task-dependent: $32{,}768$ for
code/medical/science/commonsense/summarization/math, $2{,}048$ for
SQL and NarrativeQA, $10{,}240$ for QuALITY, and $65{,}536$ for the
general mixture. Loss is computed over assistant tokens only; Liger
kernels are enabled.

\section{Squeezing Throughput from Data-Parallel Communication}
\label{sec:dp-compression}

Activation compression reduces the per-step cost of pipeline-parallel
(PP) communication, but the throughput benefit is realized only if the
\emph{data-parallel} (DP) gradient all-reduce is reduced to a
comparable level. Otherwise the DP all-reduce dominates wall-clock
time and the activation savings are invisible. This section
formalizes the trade-off, surveys the families of DP-compression
techniques compatible with our setting, and identifies the operating
points at which each becomes necessary.

\subsection{Bottleneck analysis}
\label{subsec:bottleneck}

Consider a model with $N$ parameters trained over $K$-stage pipeline
parallelism and $D$-way data parallelism, with all communication
crossing a network of bandwidth $B$ bytes/sec. Let $S$ denote the
tokens per optimizer step, $H$ the hidden size, and
$b_{\mathrm{act}}, b_{\mathrm{grad}}$ the bytes per activation and
gradient element.

\paragraph{PP communication.}
Pipeline parallelism transmits a forward activation and a backward
gradient at each of the $K{-}1$ boundaries. With activation compression
ratio $\rho_{\mathrm{act}}\in(0,1]$ (fraction of bytes still sent),
\begin{equation}
\label{eq:tpp}
T_{\mathrm{PP}}
= \frac{2\,(K{-}1)\,S\,H\,b_{\mathrm{act}}\,\rho_{\mathrm{act}}}{B}.
\end{equation}

\paragraph{DP communication.}
A standard ring all-reduce of an $N\,b_{\mathrm{grad}}$-byte gradient
costs $2(D{-}1)/D \cdot N\,b_{\mathrm{grad}}$ bytes per rank
\citep{patarasuk2009ringallreduce}. With per-step compression ratio
$\rho_{\mathrm{DP}}$ and outer-step interval $\tau$ (one all-reduce
per $\tau$ optimizer steps, as in Local SGD or DiLoCo),
\begin{equation}
\label{eq:tdp}
T_{\mathrm{DP}}
= \frac{2\,(D{-}1)/D \cdot N\,b_{\mathrm{grad}}\,\rho_{\mathrm{DP}}}{B\,\tau}.
\end{equation}

\paragraph{Step time.}
Under 1F1B-style scheduling, the per-step time is approximately
$\max(T_{\mathrm{compute}},\,T_{\mathrm{PP}},\,T_{\mathrm{DP}})$ when
phases overlap, and their sum when they serialize. Throughput in
tokens-per-second is $\mathrm{TPS}=S/\text{step time}$.

\paragraph{Concrete numbers.}
For our reference configuration ($N{=}1.2$B, $H{=}2048$, $K{=}8$,
$D{=}8$, $B{=}200$~Mbps $\approx 25$~MB/s, $S{=}16{,}384$ tokens/step,
BF16 throughout, A100 hardware), the components evaluate to:
\begin{itemize}
  \item Compute (one optimizer step across all stages, with
    $\sim 50\%$ MFU): $\sim 100$~ms.
  \item PP communication at $\rho_{\mathrm{act}}=0.05$ (95\% masking):
    $\sim 1.9$~s, dominated by per-microbatch activation transfer at
    each of $7$ boundaries in both directions.
  \item DP all-reduce uncompressed ($\rho_{\mathrm{DP}}=1$, $\tau=1$):
    $\sim 168$~s --- nearly three minutes per step.
\end{itemize}
The uncompressed DP all-reduce alone is $\sim 90\times$ larger than
the compressed PP transfer, and $\sim 1{,}700\times$ larger than the
compute window. Activation compression on its own therefore yields
no measurable speedup.

\subsection{When DP compression is essential}
\label{subsec:why-compress}

For activation compression to translate into wall-clock speedup,
$T_{\mathrm{DP}}$ must be reduced to within an order of magnitude of
$T_{\mathrm{PP}}$. From Eqs.~\eqref{eq:tpp}--\eqref{eq:tdp}, the
required combined compression along the DP axis is
\begin{equation}
\label{eq:required-dp}
\frac{\rho_{\mathrm{DP}}}{\tau}
\;\lesssim\;
\frac{(K{-}1)\,S\,H\,b_{\mathrm{act}}\,\rho_{\mathrm{act}}}
     {(D{-}1)/D \cdot N\,b_{\mathrm{grad}}}.
\end{equation}
For our configuration, this evaluates to roughly $\rho_{\mathrm{DP}}/\tau
\lesssim 1/100$, so the joint per-step compression and frequency
reduction must reach two orders of magnitude.

A single mechanism is rarely enough. Per-step gradient compressors
above $\sim 32$--$64\times$ tend to harm convergence, and outer-step
intervals $\tau > 500$ are at the edge of empirical stability for
Local SGD / DiLoCo on language models. The practical strategies
\emph{compose} a compressor with a frequency-reduction technique ---
e.g., $64\times$ PowerSGD with $\tau=20$ DiLoCo --- to multiply
their effects.

\subsection{Families of DP compression}
\label{subsec:families}

DP-compression techniques fall into four orthogonal families. Methods
within a family typically cannot be combined; methods across families
typically can.

\paragraph{Quantization (F1).}
Replace each gradient element with a low-bit representation. Examples
include FP8/INT8, signSGD~\citep{bernstein2018signsgd}, and 1-bit Adam.
Compression ranges from $2\times$ (FP16$\to$FP8) to $32\times$
(1-bit). Stochastic rounding and error
feedback~\citep{karimireddy2019error} are typically required at
aggressive bit widths.

\paragraph{Sparsification (F2).}
Transmit only the largest-magnitude entries. Top-$K$
sparsification~\citep{lin2018dgc} achieves
$50$--$100\times$ compression when paired with momentum correction
and error feedback, but the indices themselves consume bandwidth and
cap practical compression.

\paragraph{Low-rank factorization (F3).}
Approximate the gradient matrix by a low-rank product. PowerSGD
\citep{vogels2019powersgd} reaches $32$--$256\times$ compression via
rank-1 to rank-4 sketches with subspace iteration. The compressor is
linear, so the projection commutes with all-reduce; error feedback
recovers most of the lost precision.

\paragraph{Frequency reduction (F4).}
Rather than compress each all-reduce, perform fewer of them. Local
SGD~\citep{stich2019localsgd} runs $\tau$ inner SGD steps locally
between all-reduces; DiLoCo \citep{douillard2023diloco} adds an outer
Nesterov optimizer and is empirically stable up to $\tau\approx 500$
for LLM pretraining. Streaming DiLoCo
\citep{douillard2025streamingdiloco} pipelines the outer all-reduce
with inner compute, removing the synchronous barrier.

\paragraph{Composition.}
Frequency reduction multiplies with any per-step compressor. PowerSGD
at $\rho_{\mathrm{DP}}=1/64$ combined with $\tau=20$ DiLoCo yields an
effective $1280\times$ reduction in DP bytes-per-step --- enough to
clear the threshold in Eq.~\eqref{eq:required-dp} for our regime.

\subsection{Comparison and selection}
\label{subsec:comparison}

Table~\ref{tab:dp-comp} summarizes the practical operating points.

\begin{table}[ht]
\centering
\small
\setlength{\tabcolsep}{4pt}
\renewcommand{\arraystretch}{1.15}
\begin{tabular}{l c c l}
\toprule
\textbf{Method} & \textbf{Per-step ratio} & \textbf{Composes with $\tau$} & \textbf{Reported convergence cost} \\
\midrule
FP8/INT8 quantization     & $2$--$4\times$       & yes & $\leq 1\%$ loss gap \\
1-bit Adam / signSGD      & $32\times$           & yes & $1$--$3\%$ loss gap \\
Top-$K$ sparsification    & $50$--$100\times$    & yes & $1$--$5\%$ loss gap; index overhead \\
PowerSGD (rank 4)         & $32$--$256\times$    & yes & $\leq 2\%$ loss gap \\
Local SGD / DiLoCo        & --- ($\tau$-based)   & --- & $\leq 1\%$ loss gap up to $\tau{=}500$ \\
Streaming DiLoCo          & --- ($\tau$-based)   & --- & $\leq$ DiLoCo, hides comm \\
\bottomrule
\end{tabular}
\caption{DP-compression methods with demonstrated stability for
language-model training. Per-step compression refers to bytes-on-wire
relative to a full BF16 all-reduce. Frequency-reduction methods do
not have a per-step compression ratio because they reduce the number
of all-reduces, not the bytes per all-reduce; they are stacked with
F1--F3 rather than competing with them.}
\label{tab:dp-comp}
\end{table}

The selection rule that follows from Eq.~\eqref{eq:required-dp} is
straightforward:

\paragraph{High bandwidth ($\geq 1$~Gbps inter-node).}
The required compression ratio is small, and PowerSGD at
$32$--$64\times$ alone usually meets it. No frequency reduction
needed; implementation is intra-step with no algorithmic changes.

\paragraph{Low bandwidth ($\sim 200$~Mbps).}
Single-method compression tops out around $64\times$ without
convergence loss, leaving roughly an order of magnitude on the table.
Composition with DiLoCo ($\tau\geq 10$) is required. Streaming DiLoCo
additionally hides the outer all-reduce under inner compute, removing
the synchronous barrier --- this is our default configuration.

\paragraph{Federated or geo-distributed ($\leq 50$~Mbps).}
Composition is mandatory. Streaming DiLoCo with
$\tau\in[50,500]$ stacked with $\rho_{\mathrm{DP}}\leq 1/64$ is the
only regime that delivers usable throughput; convergence quality must
be re-validated at the chosen $\tau$.

\section{On the Compatibility with Dual-EMA Optimizers}
\label{app:ademamix-comparison}

A natural alternative to our spectral correction is to use the masked
gradient as a noisy first-moment signal and reach for a dual-EMA
optimizer that is known to tolerate gradient staleness, most notably
AdEMAMix~\citep{pagliardini2024ademamix}. AdEMAMix replaces Adam's
single first-moment EMA with a mixture of a \emph{fast} EMA
$m_1$ ($\beta_1\approx 0.9$) and a \emph{slow} EMA $m_2$
($\beta_3\approx 0.9999$). Its update rule is
\[
\theta_{t+1}
= \theta_t - \eta \cdot
\frac{\hat m_1^{(t)} + \alpha\, m_2^{(t)}}
{\sqrt{\hat v^{(t)}} + \epsilon},
\]
where $\hat m_1$ is bias-corrected, $\hat v$ is Adam's second moment,
and $\alpha\approx 5$ controls the weight of the slow EMA. The empirical
finding driving the design is that gradients can remain useful for tens
of thousands of steps, and a single EMA cannot simultaneously weight the
immediate past and the distant past.

This sounds superficially similar to our two-circuit design, both
maintain a slow signal alongside a fast one, so it is reasonable to
ask whether AdEMAMix could replace the spectral correction
entirely, with the masked gradient serving as both inputs to the
optimizer (the fast and slow EMAs would simply be two views of the same
noisy signal). We argue this would not work in our setting, for two
reasons.

\paragraph{(1) AdEMAMix mixes old gradients, not cleaner
gradients.}
The slow EMA $m_2$ in AdEMAMix is an exponentially weighted average of
the same gradients that feed the fast EMA. Its value comes from
\emph{temporal smoothing} of a noisy signal, not from any independent
high-fidelity source. In our setting, the fast circuit's gradients are \emph{biased away from the true gradient
subspace} by approximately $\rho\sigma_m^2$ even after spectral filtering
(Lemma~\ref{lem:noise-contract}), and unbiased only in the sense of zero
conditional mean. Smoothing $T$ such gradients with a slow EMA reduces
the variance, but does \emph{not} change the
underlying noise spectrum: directions weakly supported by the signal
remain weakly supported in the EMA. Our anchor circuit, by contrast,
provides an \emph{independent} signal computed from \emph{unmasked}
activations. Therefore, its principal subspace reflects the true gradient
geometry, not a smoothed copy of the masked-gradient noise.

\paragraph{(2) Dual EMAs cannot recover lost subspaces.}
A core claim of our analysis (Section~\ref{app:convergence}) is that
the masked gradient is approximately isotropic relative to the anchor
basis (Lemma~\ref{lem:noise-contract}, Assumption~\ref{assump:noise}).
This is why a low-rank spectral filter contracts noise quadratically in
rank. AdEMAMix has no notion of an anchor basis; it operates entirely
in the canonical parameter basis and applies the same dual-EMA dynamics
elementwise. If the gradient signal lives in a low-rank subspace
(as fine-tuning gradients
do~\citep{aghajanyan2021intrinsic, hu2022lora}) and the noise is
approximately isotropic, AdEMAMix's elementwise smoothing extracts the
signal at rate $O(1/T_{\mathrm{eff}})$ in noise variance, but discards
the structural advantage of low-rankness. Our spectral filter exploits
that structure directly. 

However, we do not claim AdEMAMix is unhelpful in our setting. In fact, the
two mechanisms are largely complementary. A natural composition is to
apply AdEMAMix's dual-EMA dynamics to the \emph{filtered} gradient
$G_t^{\mathrm{proj}}$ rather than to the raw masked gradient. The fast
EMA captures recent corrected directions, the slow EMA averages the
filter's output over many steps to further reduce residual noise, and
the spectral correction continues to enforce subspace structure. We did
not run this composition in the main experiments because our convergence
analysis targets the single-EMA AdamW baseline directly. We believe this would be an interesting future direction to explore. 

\end{document}